\documentclass{article}

\PassOptionsToPackage{numbers}{natbib}
\usepackage[preprint]{neurips_2022}
\usepackage[utf8]{inputenc} 
\usepackage[T1]{fontenc}    
\usepackage{url}            
\usepackage{booktabs}       
\usepackage{amsfonts}       
\usepackage{nicefrac}       
\usepackage{microtype}      
\usepackage{xcolor}         
\usepackage{microtype}
\usepackage{graphicx}
\usepackage{subfigure}
\usepackage{booktabs} 
\usepackage{enumerate}
\usepackage{thmtools}
\usepackage{thm-restate}
\usepackage{amsthm}
\newtheorem{theorem}{Theorem}
\newtheorem{lemma}{Lemma}
\newtheorem{corollary}{Corollary}

\newtheorem{example}{Example}
\newtheorem{assumption}{Assumption}
\declaretheorem[name=Theorem,numberwithin=section]{theorem}
\usepackage{url}
\usepackage{blt}
\usepackage{wrapfig}

\makeatletter
\usepackage[bookmarks,unicode,colorlinks=true]{hyperref}%
  \def\@citecolor{blue}%
  \def\@urlcolor{blue}%
  \def\@linkcolor{blue}%

\def\orcidID#1{\smash{\href{http://orcid.org/#1}{\protect\raisebox{-1.25pt}{\protect\includegraphics{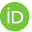}}}}}
\makeatother

\hypersetup{pdfauthor={Name}}
\usepackage{cleveref}

\title{Recursive Reinforcement Learning}

\author{%
 {~~~~Ernst Moritz Hahn}~\orcidID{0000-0002-9348-7684} \\ 
 ~~~~University of Twente
 \And
 {~~~~~~Mateo Perez}~\orcidID{0000-0003-4220-3212} \\
 ~~~~~~University of Colorado Boulder 
 \And
 {~Sven Schewe}~\orcidID{0000-0002-9093-9518}~~~~ \\
 University of Liverpool~~~~\\
 \AND
 {Fabio Somenzi}~\orcidID{0000-0002-2085-2003} \\
 University of Colorado Boulder
 \And
 {Ashutosh Trivedi}~\orcidID{0000-0001-9346-0126} \\
 University of Colorado Boulder
 \And
 {~Dominik Wojtczak}~\orcidID{0000-0001-5560-0546}~~~~ \\
 ~University of Liverpool~~~~
}

\usepackage{amsmath,amsfonts,bm}

\def\eqref#1{equation~\ref{#1}}

\def\1{\bm{1}}

\DeclareMathAlphabet{\mathsfit}{\encodingdefault}{\sfdefault}{m}{sl}
\SetMathAlphabet{\mathsfit}{bold}{\encodingdefault}{\sfdefault}{bx}{n}

\DeclareMathOperator*{\argmax}{arg\,max}

\usepackage{algorithmic}
\usepackage[linesnumbered, ruled]{algorithm2e}
\SetKwComment{Comment}{/* }{ */}
\begin{document}
\maketitle

\begin{abstract}
Recursion is the fundamental paradigm to finitely describe potentially infinite objects.
As state-of-the-art reinforcement learning (RL) algorithms cannot directly reason about recursion, they must rely on the practitioner's ingenuity in designing a suitable ``flat'' representation of the environment.
The resulting manual feature constructions and approximations are cumbersome and error-prone; their lack of transparency hampers scalability. 
To overcome these challenges, we develop RL algorithms capable of computing optimal policies in environments described as a collection of Markov decision processes (MDPs) that can recursively invoke one another.
Each constituent MDP is characterized by several entry and exit points that correspond to input and output values of these invocations.
These recursive MDPs (or RMDPs)  are expressively equivalent to probabilistic pushdown systems (with call-stack playing the role of the pushdown stack), and can model probabilistic programs with recursive procedural calls.
We introduce \emph{Recursive Q-learning}---a model-free RL algorithm for RMDPs---and prove that it converges for finite, single-exit and deterministic 
multi-exit RMDPs under mild assumptions. 
\end{abstract}

\section{Introduction}
\label{sec:intro}

Reinforcement learning~\cite{Sutton18} (RL) is a stochastic approximation based approach to optimization, where learning agents rely on scalar reward signals from the environment to converge to an optimal behavior. 
Watkins's seminal approach~\cite{watkins1989learning} to RL, known as Q-learning, judiciously combines exploration/exploitation with dynamic programming to provide guaranteed convergence~\cite{watkins1992q} to optimal behaviors in environments modeled as Markov decision processes (MDPs) with finite state and action spaces.
RL has also been applied to MDPs with uncountable state and action spaces, although convergence guarantees for such environments require strong regularity assumptions.
Modern variants of Q-learning (and other tabular RL algorithms) harness the universal approximability and ease-of-training rendered by deep neural networks~\cite{goodfellow2016deep} to discover creative solutions to problems traditionally considered beyond the reach of AI~\cite{Mnih15,vinyals2019grandmaster,Silver16}. 

These RL algorithms are designed with a flat Markovian view of the environment in the form of a ``state, action, reward, and next state'' interface~\cite{Brockm16} in every interaction with the learning agent, where the states/actions may come from infinite sets. 
When such infinitude presents itself in the form of finitely represented recursive structures, the inability of the RL algorithms to handle structured environments means that the structure present in the environment is not available to the RL algorithm to generalize its learning upon.
The work of~\cite{watkins1989learning} already provides a roadmap for hierarchically structured environments; since then, considerable progress has been made in developing algorithms for hierarchical RL~\cite{barto2003recent,dietterich2000hierarchical,sutton1999between,parr1998reinforcement} with varying optimality guarantees.
Still, the hierarchical MDPs are expressively equivalent to finite-state MDPs, although they may be exponentially more succinct
(Lemma~\ref{lem:exp_succ}).
Thus, hierarchical RL algorithms are inapplicable in the presence of unbounded recursion.

\begin{figure}
    \centering
  \begin{tikzpicture}[node distance=4cm]

    \draw(0, 0) rectangle (5.9,2.2);
    \draw (5.5, 2.4) node {$T$};

    \node[loc](u1) at (0, 1) {$u_1$};
    \node[loc](u2) at (5.9, 1) {$u_2$};

    \node[boxloc](b1) at (1.5, 1) {$b_1{:}S$};
    \node[boxloc](b2) at (2.9, 1) {$b_2{:}S$};
    \node[boxloc](b3) at (4.3, 1) {$b_3{:}S$};

    \node[port](b1n1) at (1, 1) {};
    \node[port](b1x1) at (2, 1) {};

    \node[port](b2n1) at (2.4, 1) {};
    \node[port](b2x1) at (3.4, 1) {};

    \node[port](b3n1) at (3.8, 1) {};
    \node[port](b3x1) at (4.8, 1) {};
    
    \draw[trans] (u1) -- node[below]{{\small d}} node[above]{{\small $-0.5$}} (b1n1);
    \draw[trans] (b1x1) --  (b2n1);
    \draw[trans] (b2x1) --  (b3n1);
    \draw[trans] (b3x1) --node[below]{{\small c}} node[above]{{\small $-0.5$}} (u2);
    \draw[trans] (u1) -- +(1, -0.9) -- node[above]{{$m, -8$}}  +(5.3, -0.9) -- (u2);

    \begin{scope}[xshift = 6.5 cm]
    \draw(0, 0) rectangle (3.1,2.2);
    \draw (3, 2.4) node {$S$};
    
    \node[boxloc](b4) at (1.5, 1.5) {$b_4{:}H$};
    \node[boxloc](b5) at (1.5, 0.5) {$b_5{:}S$};
    
    \node[loc](u3) at (0, 1) {$u_3$};
    \node[loc](u4) at (3.1, 1) {$u_4$};
    
    \node[probloc](p1) at (0.6, 1.5) {};
    \node[probloc](p2) at (0.6, 0.5) {};

    \draw[-] (u3) -- node[above]{{\small r, $-1.5~~~~~~~~~$}} (p1);
    \draw[-] (u3) -- node[below left]{{\small f, $-1$}}  (p2);
    
    \node[port](b4n1) at (1, 1.5) {};
    \node[port](b5n1) at (1, 0.5) {};
    \node[port](b4x2) at (2, 1.7) {};
    \node[port](b4x1) at (2, 1.3) {};
    \node[port](b5x1) at (2, 0.5) {};

    \draw[trans](p1) -- node[above]{{\tiny 0.3}} (b4n1); 
    \draw[trans](p2) -- node[below]{{\tiny 0.4}} (b5n1); 
    \draw[trans] (p1) -- +(0, -0.45) -- +(2, -0.45) -- (u4);
    \draw[trans] (p2) -- +(0, 0.45) -- +(2, 0.45) -- (u4);

    \draw[trans,thick,draw=red!50!black](b4x2) -- node[above]{{\small $+0.2$}}+(0.8, 0) --  (u4); 
    \draw[trans](b4x1) -- node[above]{{\small $+0.2$}} (u4); 
    \draw[trans](b5x1) --(u4); 

    \end{scope}
    
    \begin{scope}[xshift = 10.2cm]
    \draw(0, 0) rectangle (3.2,2.2);
    \draw (3, 2.4) node {$H$};
    
    \node[boxloc](b6) at (1, 1) {$b_6{:}H$};
    \node[boxloc](b6) at (2.2, 1) {$b_7{:}H$};

    \node[port](b6n1) at (0.5, 1) {};
    \node[port](b7n1) at (1.7, 1) {};
    \node[port](b6x1) at (1.5, 0.8) {};
    \node[port](b6x2) at (1.5, 1.2) {};
    \node[port](b7x1) at (2.7, 0.8) {};
    \node[port](b7x2) at (2.7, 1.2) {};
    
    \node[loc](u5) at (0, 1) {$u_5$};
    \node[loc](u6) at (3.2, 0.5) {$u_6$};
    \node[loc,thick,draw=red!50!black](u7) at (3.2, 1.5) {$u_7$};

    \node[probloc](p3) at (0.4, 0.1) {};
    \draw[-] (u5) -- node[pos=0.8, left]{{\small n,$-0.01$}} (p3);
    \draw[trans](p3) -- node[below right]{{\tiny 0.3}} (b6n1); 
    \draw[trans](p3) --+(2.5, 0) --  (u6);
    
    \draw[trans,thick,draw=red!50!black](u5) -- +(0.5, 0.8) --node[above]{y, {\small -$0.2$}} +(2.5, 0.8)   --  (u7);
    \draw[trans,thick,draw=red!50!black](b6x2) -- +(0.2, 0.3) --  (u7);
    \draw[trans,thick,draw=red!50!black](b7x2) --  (u7);
    
     \draw[-](b6x1) --(b7n1); 
    \draw[-](b7x1) --(u6); 
    \end{scope}
  \end{tikzpicture}
\caption{A recursive Markov decision process with three components $T$, $S$, and $H$.}
\label{fig:rmdp}
\end{figure}
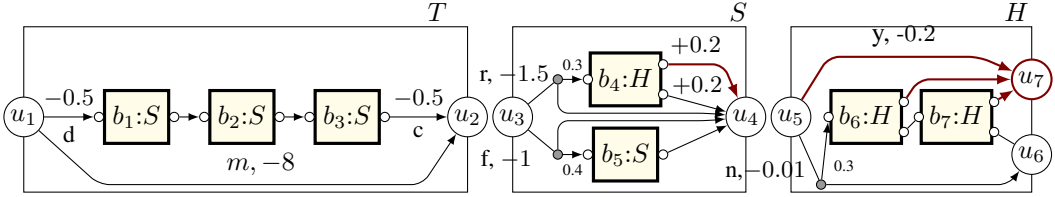
On the other hand, recursion occurs naturally in human reasoning~\cite{Corballis}, mathematics and computation~\cite{skolem1923begrundung,soare2016turing}, and physical environments~\cite{meiss2007differential}. 
Recursion is a powerful cognitive tool in enabling a divide-and-conquer strategy~\cite{cormen2022introduction} to problem solving (e.g., tower of Hanoi, depth-first search) and, consequently, recursive solutions enhance explainability in the form of intuitive inductive proofs of correctness.
Unlike flat representations, the structure exposed by recursive definitions enables generalizability. 
Recursive concepts, such as recursive functions and data structures, provide scaffolding for efficient and transparent algorithms.
Finally, the models of physical environments express the system evolution in the form of recursive equations.
We posit that the lack of RL algorithms in handling recursion is an obstacle to their applicability, explainability, and  generalizability.
This paper aims to fill the gap by studying \emph{recursive Markov decision processes}~\cite{EtessamiY12} as environment models in reinforcement learning.
We dub this setting recursive reinforcement learning.

\paragraph{MDPs with Recursion.}
A recursive Markov Decision Process (RMDP)~\cite{EtessamiY12} is a finite collection of special MDPs, called component MDPs, with special entry and exit nodes that take the role of input parameter and return value, respectively.
The states of component MDPs may either be the standard MDP states, or they may be ``boxes'' with input and output ports; these boxes are mapped to other component MDPs (possibly, the component itself) with some matching of the entry and exit nodes. 
An RMDP where every component has only one exit is called a $1$-exit RMDP, otherwise we call it a general or multi-exit RMDP.
The $1$-exit RMDPs are strictly less expressive than general RMDPs as they are equivalent to functions without any return value. 
Nonetheless, $1$-exit RMDPs are more expressive than finite-state MDPs~\cite{Put94} and relate closely to controlled branching  processes~\cite{EtessamiY15}.
\begin{example}[Cloud Computing]
\label{ex1}
As an example of recursive MDP, consider the Boolean program shown in Figure~\ref{fig:BMDP}.
This example (inspired from~\cite{Hahn21a}) models a cloud computing scenario to compute a task $T$ depicted as the component $T$.
Here, a decision maker is faced with two choices: either she can choose to execute the task monolithically (with a cost of $8$ units) or chose to decompose the task into three $S$ tasks. The process of decomposition and later combining the results cost $0.5$ units.
Each task $S$ can either be executed on a fast, but unreliable server that costs $1$ unit, but with probability $0.4$ the server may crash, and require a recomputation of the task.
When executed on a reliable server, the task $S$ costs $1.5$ units, however the task may be interrupted by a higher-priority task and the decision maker will be compensated for this delay by earning a $0.2$ unit reward.
During the interrupt service routine $H$, there is a choice to upgrade the priority of the taks for a cost of $0.2$ units. Otherwise, the interrupt service routing costs $0.01$ unit (due to added delay) and the interrupt service routine itself can be interrupted, causing the service routine to be re-executed in addition to that of the new interrupt service routine.
The goal of the RL agent is to compute an optimal policy that maximize the total reward to termination.

This example is represented as a recursive MDP in Figure~\ref{fig:rmdp}.
This RMDP has three components $T$, $S$, and $H$.
The component $T$ has three boxes $b_1$, $b_2$, and $b_3$ all mapped to components $S$.
The component $T$ and $S$ both have single entry and exit nodes, while the component $H$ has two exit nodes.
Removing the thick ({\color{red!50!black}maroon}) transitions and the exit $u_7$ makes the RMDP $1$-exit.
The edges specify both, the name of the action and the corresponding reward. While the component $T$ is non-stochastic, components $S$ and $H$ both have stochastic transitions depicted by the grey circle.

\end{example}

Recursive MDPs strictly generalize finite-state MDPs and hierarchical MDPs, and semantically encode countable state MDPs with the context encoding the stack frame of unbounded depth.
RMDPs generalize several well-known stochastic models including stochastic context-free grammars~\cite{manning1999foundations,lari1990estimation} and multi-type branching processes~\cite{harris1964theory,vatutin1993branching}.
Moreover, RMDPs are expressively equivalent to probabilistic pushdown systems (MDPs with unbounded stacks), and can model probabilistic programs with unrestricted recursion.   
Given their expressive power, it is not surprising that reachability and termination problems for general RMDPs are undecidable. 
Single-exit RMDPs, on the other hand, are considerably more well-behaved with decidable termination~\cite{EtessamiY12} and total reward optimization under positive reward restriction~\cite{EtessamiWY19}.
Exploiting these properties, a convergent RL algorithm~\cite{Hahn21a} has been proposed for $1$-exit RMDPs with positive reward restriction. However, to the best of our knowledge, no RL algorithm exists for general RMDPs.
\begin{figure}
\newsavebox\testboxA
\newsavebox\testboxB
\newsavebox\testboxC

\setbox\testboxA=\hbox{%
    \begin{lstlisting}[language=Python,frame=none]
      def T():
        a =  %?%(%$\{\tt mono, divide\}$%);
        if a = mono: 
          execute_mono() # -$8
        else :
          decompose() # -$0.5
          S()
          S()
          S()
          combine()  # -$0.5
        return
    \end{lstlisting}
}
\setbox\testboxB=\hbox{%
    \begin{lstlisting}[language=Python,frame=none]
      def S(): 
        a =  %?%(%$\{\tt reliable, fast\}$%);
        if a = fast:
          try:
            execute_fast() # -$1 
          except Crash: # prob = 0.4
            S()  
        else:  
          try:
            execute_reliable() #-$1.5
          except Interrupt: #prob = 0.3
            H() #   +$0.2
        return    
  \end{lstlisting}
}
\setbox\testboxC=\hbox{%
    \begin{lstlisting}[language=Python,frame=none]
      def H():
       upgrade =  %?%(%$\{\tt yes, no\}$%);
       if upgrade = no: 
         try:
           ISR() # -$0.01 
         except Interrupt: #prob = 0.3
           if (H()):
             return 1
           if (H())
             return 1
         return 0
       else :  #-$0.2
         return 1 
  \end{lstlisting}
}
\centering
\begin{adjustbox}{minipage=0.32\linewidth}
\subfigure[Task $T$]{\usebox{\testboxA}} 
\end{adjustbox}
\hfill
\begin{adjustbox}{minipage=0.32\linewidth}
\subfigure[Task $S$]{\usebox{\testboxB}} 
\end{adjustbox}
\hfill
\begin{adjustbox}{minipage=0.33\linewidth}
\subfigure[Task $H$]{\usebox{\testboxC}}
\end{adjustbox}
  \caption{\footnotesize A probabilistic Boolean program sketch where the choice of the hole ($?$) is to be filled by RL agent.}
\label{fig:BMDP}
\end{figure}

\paragraph{Applications of Recursive RL.} 
Next, we present some paradigmatic applications of recursive RL.
\begin{itemize}
    \item {\bf Probabilistic Program Synthesis.} As shown in Example~\ref{ex1}, RMDPs can model procedural probabilistic Boolean program. Hence, the recursive RL can be used for program synthesis in unknown, uncertain environments.  Boolean abstractions of programs \cite{ball2000bebop} are particularly suited to modeling as RMDPs.
    Potential applications include program verification~\cite{ball2000bebop,ball2000slam,esparza2000efficient} and program synthesis~\cite{griesmayer2006repair}.
    
    \item {\bf Context-Free Reward Machines.} Recently, reward machines have been proposed to model non-Markovian reward signals in RL. In this setting, a regular language extended with the reward signals (Mealy machines) over the observation sequences of the MDP is used to encode reward signals. 
    In this setting the RL algorithms operate on the finite MDP given by the product of the MDP with the reward machine. 
    Following the Chomsky hierarchy, context-free grammars or pushdown automata can be used to provide more expressive reward schemes than regular languages.
    As an example of such a more expressive reward language, consider a grid-world with a reachability objective with some designated charging stations, where $1$-unit dwell-time charges the battery by $1$-unit.
    If every action discharges the battery by $1$-unit, the reward scheme to reach the target location without ever draining the battery cannot be captured by a regular reward machine.
    On the other hand, this reward signal can be captured with an RMDP, where charging by $1$-unit amounts to calling a component and discharging amounts to returning from the component such that the length of the call stack encodes the battery level.
    More generally, any context-free requirement over finite-state MDPs can be captured using general RMDPs.
\item  {\bf Stochastic Context-Free Grammars.} Stochastic CFGs and branching decision processes can capture a structured evolution of a system. These can be used for modeling disease spread, population dynamics, and natural languages. RRL can be used to learn optimal behavior in systems expressed using such stochastic grammars.
\end{itemize}

\paragraph{Overview.} After a motivating example, we begin the technical presentation by providing the formal definition of RMDPs and the total reward problem. We then show undecidability for the general problem. In Section~\ref{sec:deepMulti}, we develop PAC learning results for a mildly restricted setting. In Section~\ref{sec:deepSingle}, we focus on the single-exit setting and introduce Recursive Q-learning.

\section{Recursive Markov Decision Processes}
\label{sec:rmdp}

An MDP $\Mm$ is a tuple $(A, S, T, r)$ where
$A$ is a finite set of {\it actions}, $S$ is a finite set of states, 
$T : S \times A \to \DIST(S)$ is the probabilistic transition function, and $r: S \times A \to \Real$ is the reward function. 
We say that an MDP $\Mm$ is finite if both $S$ and $A$ are finite. For any state $s \in S$, $A(s)$ denotes the set of actions that may be selected in
state $s$.

A recursive Markov decision process (RMDP)~\cite{EtessamiY12} is a tuple $M = (M_1, \ldots ,M_k)$, where
each {\em component \/}
$M_i = (A_i, N_i , B_i, Y_i, \En_i, \Ex_i, \delta_i)$  consists of:
\begin{itemize} 
\item   
    A set $A_i$ of actions;
\item  
    A set $N_i$ of {\em nodes\/}, with a distinguished subset $\En_i$
    of {\em entry} nodes and a (disjoint) subset $\Ex_i$ of {\em exit} nodes (we assume an arbitrary but fixed ordering on $\Ex_i$ and $\En_i$); 
\item 
    A set $B_i$ of {\em boxes\/} along with a mapping $Y_i: B_i  \mapsto \{ 1, \ldots,k \}$
that assigns to every box (the index of) a component.
To each box $b \in B_i$, we associate a set of {\em call ports}, $\Call_b =
\{ (b,en) \mid en \in \En_{Y(b)} \}$,
and a set of {\em return ports}, $\return_b = \{ (b,ex) \mid ex \in \Ex_{Y(b)}\}$;
\item 
we let $\Call^i = \cup_{b \in B_i} \Call_b$,
$\return^i = \cup_{b \in B_i} \return_b$,
and let $Q_i =N_i \cup \Call^i \cup \return^i$ be the set of all nodes, call ports and
return ports; we refer to these as the {\em vertices} of component $M_i$.

\item 
A transition function $\delta_i : Q_i \times A_i \to \DIST(Q_i)$, where, for each
tuple $\delta_i(u, a)(v) = p$ is the transition probability of a transition from the source $u \in (N_i \setminus \Ex_i) \cup \return^i$ to the destination $v \in (N_i \setminus \En_i) \cup \Call^i$; we often write $p(v | u, a)$ for $\delta_i(u, a)(v)$.
\item 
A reward function $r_i: Q_i \times A_i \to \Real$ is the reward associated with transitions.
\end{itemize}

We assume that the set of boxes $B_1,\ldots, B_k$ and set of nodes $N_1, N_2, \ldots, N_k$ are mutually
disjoint.
We use symbols $N, B, A, Q, \delta$ to denote
the union of the corresponding symbols over all components. We say that an RMDP is finite if $k$ and all $A_i$, $N_i$ and $B_i$ are finite.

An execution of an RMDP begins at an entry node of some component and, depending
upon the sequence of input actions, the state evolves naturally like an MDP according to the transition distributions. 
However, when the execution reaches an entry port of a box, this box is stored on a stack of pending calls, and the execution continues naturally from the corresponding entry node of the component mapped to that box. 
When an exit node of a component is encountered, and if the stack of pending calls is empty then the run terminates; otherwise, it pops the box from the top of the stack and jumps to the exit port of the just popped box corresponding to the just reached exit of the component. 
The semantics of an RMDP is an infinite state MDP, whose states are pairs consisting of a sequence of boxes, called the context, mimicking the stack of pending calls and the current vertex.

The height of the call stack is incremented (decremented) when a call (return) is made. A stack height of $0$ refers to termination, while the empty stack has height $1$.

The semantics of a recursive MDP $M = (M_1, \ldots, M_k)$ with 
$M_i = (A_i, N_i , B_i, Y_i, \En_i, \Ex_i, \delta_i,r_i)$ are given as a (infinite-state) MDP $\sem{M} = (A_M, S_M, T_M, r_M)$ where 
\begin{itemize}
\item
  $A_M = \cup_{i=1}^{k} A_i$ is the set of actions;
\item
  $S_M \subseteq B^* {\times} Q$ is the set of states, consisting of the stack context and the current node;
\item
  $T_M : S_M {\times} A_M \to \DIST(S_M)$ is the transition
  function such that for $s = (\sseq{\kappa}, q) \in S_M$ and
  action $a \in A_M$, the distribution $\delta_M(s, a)$ is defined as:
  \begin{enumerate}
  \item
    if the vertex $q$ is a call port, i.e. $q = (b, en) \in \Call$,
    then  $\delta_M(s, a)(\sseq{\kappa, b}, en) = 1$;
  \item
    if the vertex $q$ is an exit node, i.e. $q = ex \in \Ex$,
    then if $\kappa = \sseq{\emptyset}$ then the process terminates
    and otherwise $\delta_M(s, a)(\sseq{\kappa'}, (b, ex)) = 1$ where
    $(b, ex) \in \return(b)$ and $\kappa = \sseq{\kappa', b}$;
  \item
    otherwise, 
    $\delta_M(s, a)(\sseq{\kappa}, q') = \delta(q, a)(q')$.
  \end{enumerate}
 \item the reward function $r_M: S_M \times A_M \to \Real$ is such that
 for $s = (\sseq{\kappa}, q) \in S_M$ and
  action $a \in A_M$, the reward $r_M(s, a)$ is zero if $q$ is either a call port or the exit node, and otherwise $r_M(s, a)(\sseq{\kappa}, q') = r(q, a)(q')$.
  We call the maximum value of the absolute one-step reward the \emph{diameter} of an RMDP
  and denote it by $r_{max} = \max_{s,a} |r(s,a)|$.
\end{itemize}

Given the semantics of an RMDP $M$ as an (infinite) MDP $\sem{M}$, the concepts of strategies as well as positional strategies are well defined.
We distinguish a special class of strategies---called {\em stackless strategies}---that are deterministic and do not depend on the history or the stack context at all. 

We are interested in computing strategies $\sigma$ that maximize the \emph{expected total reward}. Given RMDP $M$, a strategy $\sigma$ determines sequences $X_i$ and $Y_i$ of random variables denoting the $i^{th}$ state and action of the MDP $\sem{M}$. The total reward under strategy $\sigma$ and its optimal value are respectively defined as
\begin{align*}
\ETotal^M_\sigma(s) &= \lim_{N \to \infty} \eE^M_\sigma(s) 
\Bigl\{\sum\nolimits_{1 \leq i \leq N} r(X_{i-1}, Y_i) \Bigr\},
&
\ETotal^M(s) &= \sup_\sigma \ETotal^M_\sigma(s).
\end{align*}

For an RMDP $M$ and a state $s$, a strategy $\sigma$ is called {\em proper}
if the expected number of steps taken by $M$ before termination when  starting at $s$ is finite. 
To ensure that the limit above exists, as the sum of rewards can otherwise oscillate arbitrarily, we assume the following.
\begin{assumption}[Proper Policy Assumption]
\label{asm:proper}
All strategies are proper for all states.
\end{assumption}
We call an RMDP that satisfies Assumption~\ref{asm:proper} a {\em proper RMDP}. This assumption is akin to proper policy assumptions~\cite{bertsekas1991analysis} often posed on the stochastic shortest path problems, and ensures that the total expected reward is finite.
The expected total reward optimization problem over proper RMDPs subsumes the discounted optimization problem over finite-state MDPs since discounting with a factor $\lambda$ is analogous to terminating with probability $1{-}\lambda$ at every step~\cite{Sutton18}. 
The properness assumption on RMDPs can be enforced by introducing an appropriate discounting (see Appendix~\ref{sec:PAC}).

\paragraph{Undecidability.}
Given an RMDP $M$, an initial node $v$, and a threshold $D$, the \emph{strategy existence problem} is to decide whether there exists a strategy in $\sem{M}$ with value greater than or equal to $D$ when starting at the initial state $(\sseq{\emptyset}, q)$, i.e., at some entry node $q$ with an empty context.

\begin{restatable}[Undecidability of the Strategy Existence Problem]{theorem}{undec}
\label{thm:totalundec}
Given a proper RMDP 
and a threshold $D$, deciding whether there exists a strategy with expected value greater than $D$ is undecidable.
\end{restatable}

\paragraph{PAC-learnability.}
Although it is undecidable to determine whether or not a strategy can exceed some threshold in a proper RMDP, the problem of $\varepsilon$-approximating the optimal value is decidable when parameters $c_o$, $\lambda$ and $b$ (defined below) are known. Our approach to PAC-learnability~\cite{agarwal2019reinforcement} is to learn the distribution of the transition function $\delta$ well enough and then produce an approximate, but not necessarily efficient, evaluation of our learned model.

To allow PAC-learnability, we need a further nuanced notion of $\varepsilon$-proper policies. 
A policy is called $\varepsilon$-proper, if it terminates with a uniform bound on the expected number of steps for all $\Mm'$ that differ from $\Mm$ only in the transition function, where
$\sum_{q\in S, a \in A, r\in S}|\delta_\Mm(q,a)(r) - \delta_{\Mm'}(q,a)(r)| \leq \varepsilon$ (we then say that $\Mm'$ is $\varepsilon$-close to $\Mm$), and where the support of $\delta_{M'}(q,a)$ is a subset of the support of $\delta_{M}(q,a)$ for all $q \in S$ and $a \in A$.
An RMDP is called $\varepsilon$-proper, if all strategies are $\varepsilon$-proper for $M$ for all states of the RMDP.

\begin{assumption}[PAC-learnability]
We restrict our attention to $\varepsilon$-proper RMDPs. 
We further require that all policies
have a falling expected stack height.
Namely, we require for all $\Mm'$ $\varepsilon$-close to $\Mm$ and all policies $\sigma$ that the expected stack height in step $k$ is bounded by some function $c_o - \mu \cdot \sum_{i=1}^k p_{\mathsf{run}}^{\Mm_\sigma'}(k)$, where $c_o \geq 1$ is an offset, $\mu \in ]0,1]$ is the decline per step, and $p_{\mathsf{run}}^{\Mm_\sigma'}(k)$ is the likelihood that the RMDP $\Mm'$ with strategy $\sigma$ is still running after $k$ steps.
We finally require that the absolute expected value from every strategy is bounded: $\bigl|\ETotal^{\Mm'}_\sigma\bigl((\langle\emptyset\rangle,q)\bigr)\bigr|\leq b$ for some $b$.
\end{assumption}

\begin{restatable}{theorem}{totalPAC}
\label{theo:totalPAC}
For every $\varepsilon$-proper RMDP with parameters $c_o$, $\mu$, and $b$, $\ETotal^\Mm(s)$ is PAC-learnable.
\end{restatable}

These parameters can be replaced by discounting.
Indeed, our proofs start with discounted rewards, and then relax the assumptions to allow for using undiscounted rewards. Using a discount factor $\lambda$ translates to parameters $b = \frac{d}{1-\lambda}$, $c_o = 1 + \frac{1}{1-\lambda}$, and $\mu = 1 {-} \lambda$.

\section{Recursive Q-Learning for Multi-Exit RMDPs}
\label{sec:deepMulti}

While RMDPs with multiple exits come with undecidability results, they are the interesting cases as they represent systems with an arbitrary call stack. 
We suggest an abstraction that turns them into a fixed size data structure, which is well suited for neural networks.

Given a proper recursive MDP $M = (M_1, \ldots, M_k)$ with 
$M_i = (A_i, N_i , B_i, Y_i, \En_i, \Ex_i, \delta_i,r_i)$ with semantics $\sem{M} = (A_M, S_M, T_M, r_M)$, the optimal total expected reward can be captured by the following equations $\textrm{OPT}_{\tt recur}(M)$. For every $\kappa \in B^*$ and $q \in Q$:
\begin{eqnarray*}
y(\sseq{\kappa}, q) &=& 
\begin{cases}
    y(\sseq{\kappa, b}, en) & q {=} (b, en) \in \Call\\
    0 & q \in \Ex, \kappa = \sseq{\emptyset} \\
   y(\sseq{\kappa'}, (b, q)) & q \in \Ex, (b, q) \in \return(b), \kappa {=} \sseq{\kappa', b}\\
  \max\limits_{a \in A(q)} \Bigl\{ r(q, a) {+} \sum\limits_{q' \in Q}  p(q' | q, a) y(\sseq{\kappa}, q')\Bigr\} & \text{otherwise.}
\end{cases}
\end{eqnarray*}

These equations capture the optimality equations on the underlying infinite MDP $\sem{M}$. It is straightforward to see that, if these equations admit a solution, then the solution equals the optimal total expected reward~\cite{Put94}.
Moreover, an optimal policy can be computed by choosing the actions that maximize the right-hand-side. 
However, since the state space is countably infinite and has an intricate structure, an algorithm to compute a solution to these equations is not immediate.

To make it accessible to learning, we \emph{abstract} the call stack $\sseq{\kappa, b}$ to its \emph{exit value}, i.e. the total expected reward from the exit nodes of the box $b$, under the stack context $\sseq{\kappa}$.
Note that when a box is called, the value of each of its exits may still be unknown, but it is (for a given strategy) fixed.
Naturally, if two stack contexts $\sseq{\kappa, b}$  and $\sseq{\kappa', b}$ achieve the same expected total reward from each exit of the block $b$, then both the optimal strategy and the expected total reward, are the same.
 
This simple but precise and effective abstraction of stacks with exit values allows us to consider the following optimality equations $\textrm{OPT}_{\tt cont}(M)$. For every $1 \leq i \leq k$, $q \in Q_i$, ${\bf v} \in \Real^{|\Ex_i|}$:
\begin{eqnarray*}
x({\bf v}, q) &=& 
\begin{cases}
    x({\bf v'}, en)[{\bf v}' \mapsto (x({\bf v},q'))_{q'\in \return_b}] & q = (b, en) \in \Call\\
    {\bf v}(q) & q \in \Ex \\
\max\limits_{a \in A(q)} \Bigl\{ r(q, a) {+} \sum\limits_{q' \in Q}  p(q' | q, a) x({\bf v}, q')\Bigr\} & \text{otherwise.}
\end{cases}
\end{eqnarray*}
Here ${\bf v}$ is a vector where ${\bf v}(ex)$ is the 
(expected) reward that we get once we reach exit $ex$ of the current component. 
Informally when a box is called, this vector is being updated with the current estimates of the reward that we get once the box is exited. The $ex$ entry of this vector ${\bf v'} = (x({\bf v}, q'))_{q'\in \return_b}$ is $x({\bf v}, (b,ex))$, which is the value that we achieve from exit $(b,ex)$.

This continuous-space abstraction of the countably infinite space of the stack contexts enables the application deep feedforward neural networks~\cite{goodfellow2016deep} with a finite state encoding in RL.
It also provides an elegant connection to the smoothness of differences to exit values: if all exit costs are changed by less than $\varepsilon$, then the cost of each state within a box changes by less than $\varepsilon$, too.
The following theorem connects both versions of optimality equations.
\begin{restatable}[Fixed Point]{theorem}{fixedpc}
\label{thm:fix-point-correspondence}
If $y$ is a fixed point of  $\textrm{OPT}_{\tt recur}$ and $x$ is a fixed point of $\textrm{OPT}_{\tt cont}$, then $y(\sseq{\emptyset},q) = x({\bf 0},q)$. 
Moreover, any policy optimal from $({\bf 0},q)$ is also optimal from $(\sseq{\emptyset},q)$.
\end{restatable}

\begin{algorithm}[t]
\caption{Recursive Q-learning}\label{alg:rql}
Initialize $Q(s,v,a)$ arbitrarily

\While{not converged}{
    $v \gets \mathbf{0}$

    $\text{stack} \gets \emptyset$

    Sample trajectory $\tau \sim \{(s,a,r,s'), ...\}$
    
    \For{$s,a,r,s'$ in $\tau$}{
    
    Update $\alpha_i$ according to learning rate schedule
    
    \uIf{entered box}{
    $\{s_{\text{exit}_1}, \ldots, s_{\text{exit}_n}\} \gets \text{getExits}(s')$ 
    
    $v' \gets [\max_{a' \in A(s_{\text{exit}_1})} Q(s_{\text{exit}_1}, v, a'), \ldots, \max_{a' \in A(s_{\text{exit}_n})} Q(s_{\text{exit}_n}, v, a')]$
    
    $v'_{\min{}} \gets \min(v')$
    
    $v' \gets v' - v'_{\min{}}$
    
    $Q(s,v,a) \gets (1-\alpha_i) Q(s,v,a) + \alpha_i(r + \max_{a' \in A(S')} Q(s', v', a') + v'_{\min{}})$
    
    stack.push($v$)
    
    $v \gets v'$
    }
    \uElseIf{exited box}{
    $\{s_{\text{exit}_1}, \ldots, s_{\text{exit}_n}\} \gets \text{getExits}(s)$
    
    Set $k$ such that $s' = s_{\text{exit}_k}$
    
    $Q(s,v,a) \gets (1-\alpha_i) Q(s,v,a) + \alpha_i(r + v(k))$
    
    $v \gets$ stack.pop()
    }
    \Else{
    $Q(s,v,a) \gets (1-\alpha_i) Q(s,v,a) + \alpha_i(r +  \max_{a \in A(S')} Q(s', v, a'))$
    }
    }
}

\Return $Q$
\end{algorithm}

We design a generalization of the Q-learning algorithm~\cite{watkins1992q} for recursive MDPs based on the optimality equations $\textrm{OPT}_{\tt cont}$ shown in Algorithm~\ref{alg:rql}.
We implement several optimizations in our algorithm. 
We assume implicit transitions from the entry and exit ports of the box to the corresponding entry and exit nodes of the components.
A further optimization is achieved by applying a dimension reduction on the representation of the exit value vector ${\bf v}$ by normalizing these values in such a way that one of the exits has value $0$.
This normalization does not affect optimal strategies as, when two stacks incur similar costs in that they have the same offset between the cost of each exit, the optimal strategy is still the same, with the difference in cost being this offset.

While the convergence of Algorithm~\ref{alg:rql} is not guaranteed for the general multi-exit RMDPs, the algorithm converges for the special cases of deterministic proper RMDPs and $1$-exit RMDPs (Section~\ref{sec:deepSingle}).
For the deterministic multi-exit case, the observation is straightforward as the properness assumption reduces the semantics to be directed acyclic graph, and the correct values are eventually propagated from the leafs backwards.\begin{theorem}
Tabular Recursive Q-learning converges to the optimal values for deterministic proper multi-exit RMDPs with a learning rate of $1$ when all state-action pairs are visited infinitely often.
\end{theorem}

\section{Convergence of Recursive Q-Learning for Proper 1-exit RMDPs}
\label{sec:deepSingle}
Recall that a proper $1$-exit RMDP is a proper RMDP where, for each component $M_i$, the set of exits $\Ex_i$ is a singleton.
For this special case, we show that the recursive Q-learning algorithm converges to the optimal strategy.
The optimality equations $\textrm{OPT}_{\tt cont}(M)$ (similar to~\cite{EtessamiWY19}) can be simplified in the case of $1$-exit RMDPs whose unique fixed point solution will give us the optimal values of the total reward objective.
For every $q \in Q$:
\begin{eqnarray*}
x(q) &=& 
\begin{cases}
    x(en) + x (b,ex')
    & q = (b, en) \in \Call, ex = \Ex_{Y(b)}\\
\max_{a\in A(q)} \Bigl\{r(q,a) + \sum\limits_{q' \in Q} p(q'|q, a) x(q') \Bigr\}
& \text{otherwise.}
\end{cases}
\end{eqnarray*}

We now denote the system of all these equations in a vector form as $\bfx = F(\bfx)$.  
Given a 1-exit RMDP, we can easily construct its associated equation system above in linear time.

\begin{restatable}[Unique Fixed Point]{theorem}{uniqfp}
\label{thm:unique-fixed-point}
The vector consisting of  $\ETotal^M(q)$ values is the unique fixed point of $F$.
Moreover, a solution of these equations provide optimal stackless strategies.
\end{restatable}

Note that for the $1$-exit setting, Algorithm~\ref{alg:rql} simplifies to Algorithm~\ref{alg:rql1e} since $v$ is always $0$ and $v_{\min{}}$ is always the maximum Q-value for the exit. 
The convergence of the recursive Q-learning algorithm for $1$-exit RMDPs follows from Theorem \ref{thm:unique-fixed-point} and stochastic approximation~\cite{watkins1992q,borkar2000ode}.

\begin{theorem}
\label{thm:1exit-qlearn}
Algorithm~\ref{alg:rql1e} converges to the optimal values in $1$-exit RMDP when the learning rates satisfy $\sum_{i=0}^\infty \alpha_i = \infty$, $\sum_{i=0}^\infty \alpha^2_i < \infty$, and all state-action pairs are visited infinitely often.
\end{theorem}

In order to show efficient PAC learnability for $\epsilon$-proper 1-exit RMDP $M$, it suffices to know an upper bound on the expected number of steps taken by $M$ when starting at any vertex with the empty stack content, which will be denoted by $K$. 

\begin{restatable}[Efficient PAC Learning for 1-Exit RMDPs]{theorem}{oneexitPAC}
\label{thm:1exit-PAC}
For every $\epsilon$-proper $1$-exit RMDP with diameter $r_{max}$ and the expected time to terminate $\leq K$, $\ETotal^\Mm(s)$ is efficiently PAC-learnable.
\end{restatable}

\begin{algorithm}[t]
\caption{Recursive Q-learning (1-exit special case)}\label{alg:rql1e}
Initialize $Q(s,a)$ arbitrarily

\While{not converged}{

    Sample trajectory $\tau \sim \{(s,a,r,s'), ...\}$
    
    \For{$s,a,r,s'$ in $\tau$}{
    
    Update $\alpha_i$ according to learning rate schedule
    
    \uIf{entered box}{
    $s_{\text{exit}} \gets \text{getExit}(s')$
    
    $Q(s,a) \gets (1-\alpha_i) Q(s,a) + \alpha_i(r + \max_{a' \in A(S')} Q(s', a') + \max_{a' \in A(s_{\text{exit}})} Q(s_{\text{exit}}, a'))$
    }
    \uElseIf{exited box}{
    $Q(s,a) \gets (1-\alpha_i) Q(s,a) + \alpha_i(r)$
    }
    \Else{
    $Q(s,a) \gets (1-\alpha_i) Q(s,a) + \alpha_i(r +  \max_{a \in A(S')} Q(s', a'))$
    }
    }
}

\Return $Q$

\end{algorithm}

\section{Experiments}
\label{sec:expt}

We implemented Algorithm~\ref{alg:rql} in tabular form as well as with a neural network. For the tabular implementation, we quantized the vector $v$ directly after its computation to ensure that the Q-table remains small and discrete. For the neural network implementation we used the techniques used in DQN~\cite{Mnih15}, replay buffers and target networks, for additional stability. The details of this implementation can be found in the appendix.

We consider three examples: one to demonstrate the application of Recursive Q-learning for synthesizing probabilistic programs, one to demonstrate convergence in the single-exit setting, and one to demonstrate the use of a context-free reward machine. We compare Recursive Q-learning to Q-learning where the RMDP is treated as an MDP by the agent, i.e., the agent treats stack calls and returns as if they were normal MDP transitions.

\subsection{Cloud computing}
The cloud computing example, introduced in Example~\ref{ex1}, is a recursive probabilistic program with decision points for an RL agent to optimize over. The optimal strategy is to select the reliable server and to never upgrade. This strategy produces an expected total reward of $-5.3425$. Figure~\ref{fig:curves} shows that Recursive Q-learning quickly converges to the optimal solution while Q-learning oscillates around a suboptimal policy.

\begin{wrapfigure}[13]{r}{0.265\textwidth}
    \centering
    \vspace{-0.43cm}
    \includegraphics[scale=0.168]{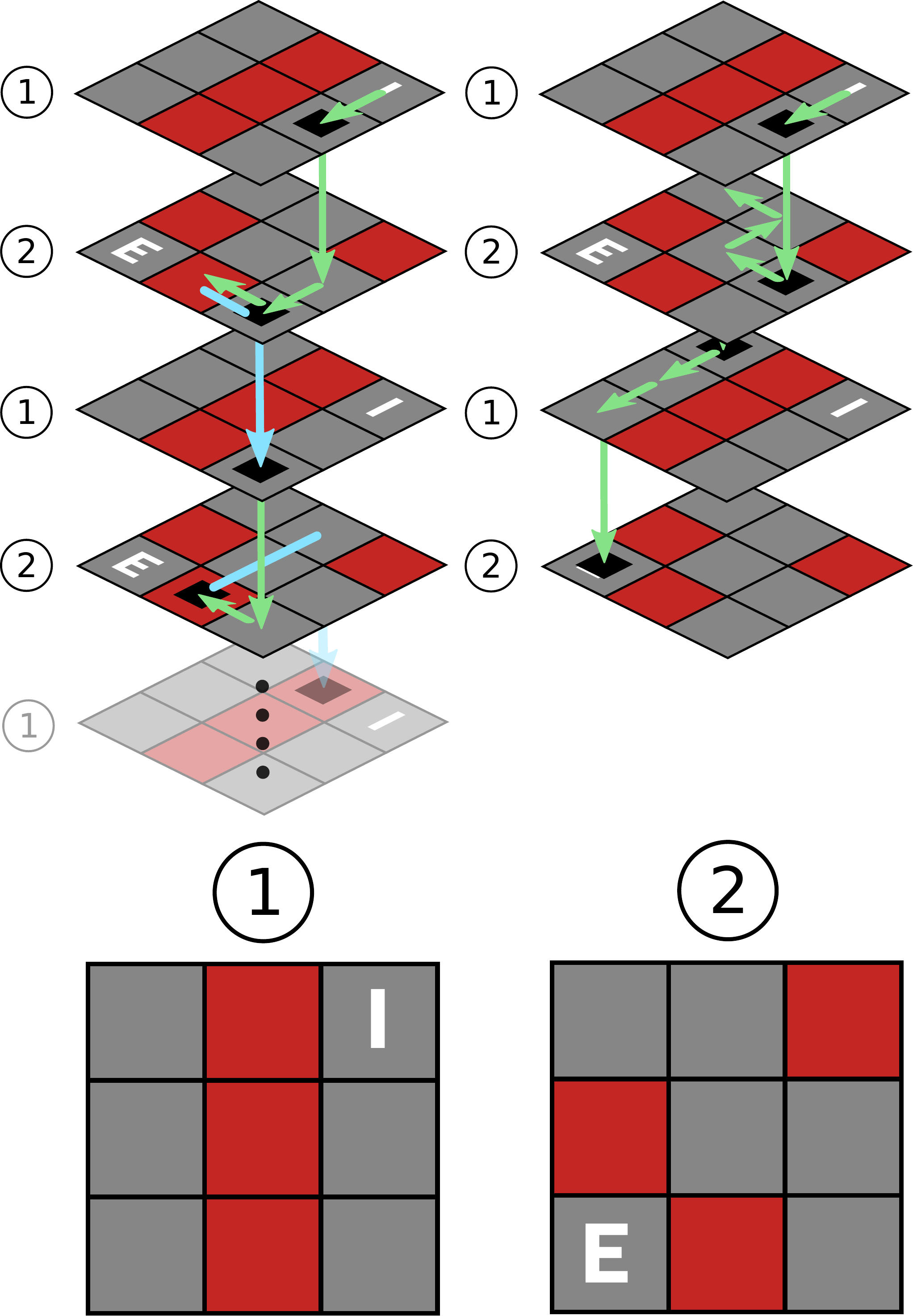}
\end{wrapfigure}
\subsection{Infinite spelunking}
Consider a single-exit RMDP gridworld with two box types, $1$ and $2$, shown at the bottom of the figure to the right. These box types are the two types of levels in an infinitely deep cave. When falling or descending to another level, the level type switches. Passing over a trap, shown in red, results in the agent teleporting to a random position and falling with probability $0.5$. 

The agent has fallen into the cave at the position denoted by $I$ without climbing equipment. 
However, there is climbing equipment in one of the types of levels at a known location denoted by $E$. The agent has four move directions---north, east, south, west---as well as an additional action to descend further or ascend. Until the climbing equipment is obtained, the agent can only descend. Once the climbing equipment is obtained, the traps no longer affect the agent and the agent can ascend only from the position where it fell down. With probability $0.01$ the agent ascends from the current level with the climbing gear. This has the effect of box-wise discounting with discount factor $0.99$. The agent's objective is to leave the cave from where it fell in as as soon as possible. The reward is $-1$ on every step.

There are two main strategies to consider. The first strategy tries to obtain the climbing gear by going over the traps. This strategy leads to an unbounded number of possible levels since the traps may repeatedly trigger. The second strategy avoids the traps entirely. The figure to the right shows partial descending trajectories from these strategies, with the actions shown in green, the trap teleportations shown in blue, and the locations the agent fell down from are shown as small black squares. Which strategy is better depends on the probability of the traps triggering. With a trap probability of $0.5$, the optimal strategy is to try and reach the climbing equipment by going over the traps. Figure~\ref{alg:rql} shows the convergence of Recursive Q-learning to this optimal strategy while the strategy learned by Q-learning does not improve.

\begin{figure}[t]
    \centering
    \begin{minipage}{0.34\textwidth}
    \includegraphics[scale=0.32]{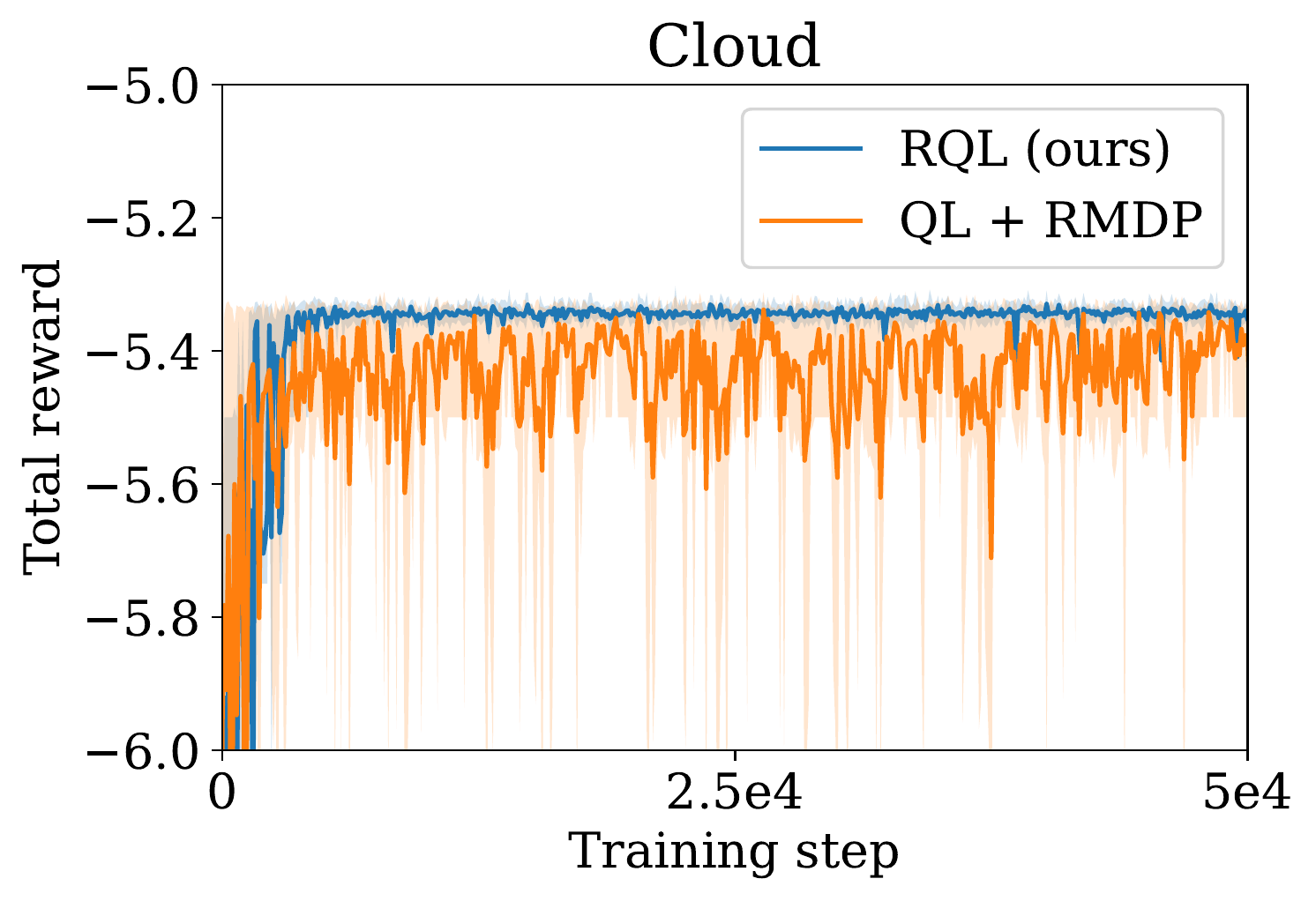}
    \end{minipage}%
    \begin{minipage}{0.33\textwidth}
    \includegraphics[scale=0.32]{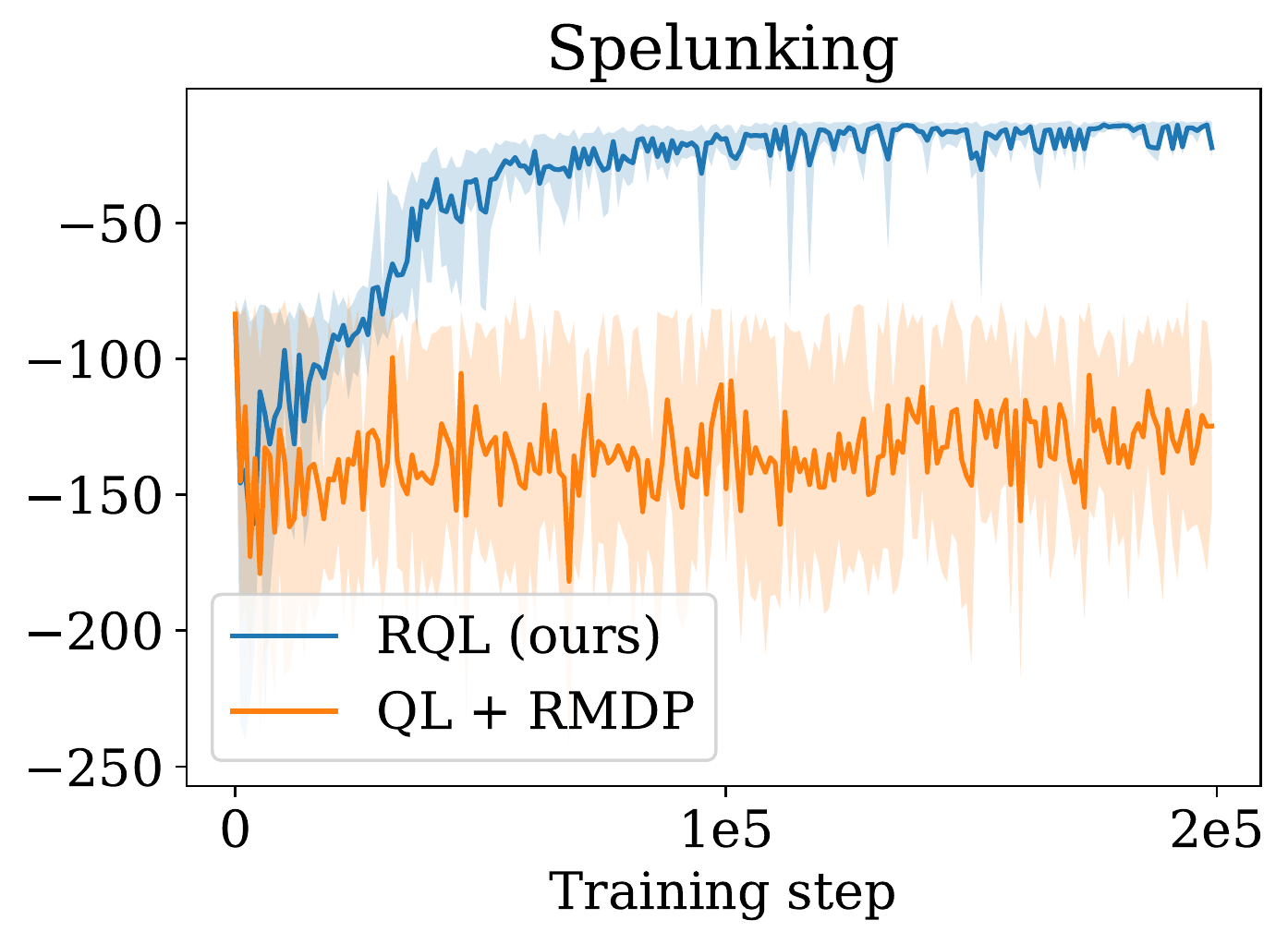}
    \end{minipage}%
    \begin{minipage}{0.33\textwidth}
    \includegraphics[scale=0.32]{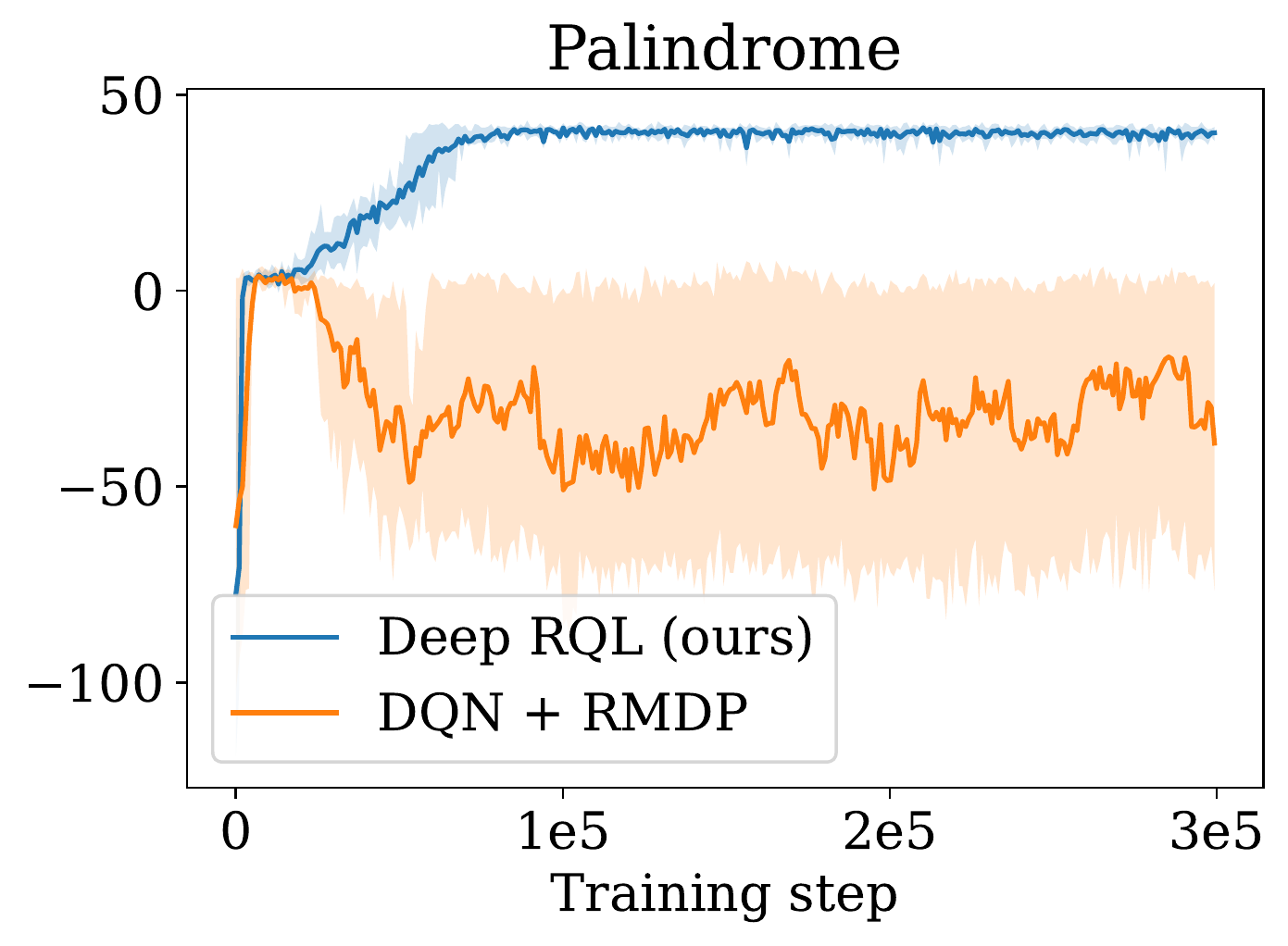}
    \end{minipage}
    \caption{Learning curves. Ten runs were performed for each method. The mean is shown as a dark line. The region between the $10^\text{th}$ and $90^\text{th}$ percentile is shaded. RQL refers to tabular Recursive Q-learning, QL + RMDP refers to using tabular Q-learning on the RMDP, Deep RQL refers to Deep Recursive Q-learning, and DQN + RMDP refers to using Deep Q-networks on the RMDP.}
    \label{fig:curves}
\end{figure}

\subsection{Palindrome gridworld}
To demonstrate the ability to incorporate context-free objectives, consider a $3 \times 3$ gridworld with a goal cell in the center and a randomly selected initial state. The agent has four move actions---north, east, south, west---and a special control action. The objective of the agent is to reach the goal cell while forming an action sequence that is a palindrome of even length. What makes this possible is that when the agent performs an action that pushes against a wall of the gridworld, no movement occurs. To monitor the progress of the property, we compose this MDP with a nondeterministic pushdown automaton. The agent must use its special action to determine when to resolve the nondeterminism in the pushdown automaton. Additionally, the agent uses its special action to declare the end of its action sequence. To ensure properness, the agent's selected action is corrupted into the special action with probability $0.01$. The agent is given a reward of $50$ upon success, $-5$ when the agent selects an action that causes the pushdown automaton to enter a rejecting sink, and $-1$ on all other timesteps.

Figure~\ref{fig:curves} shows the convergence of Deep Recursive Q-learning to an optimal strategy on this example, while DQN fails to find a good strategy.

\section{Conclusion}
\label{sec:conclusion}

Reinforcement learning so far has primarily considered Markov decision processes (MDPs). Although extremely expressive, this formalism may require ``flattening'' a more expressive representation that contains recursion. In this paper we examine the use of recursive MDPs (RMDPs) in reinforcement learning---a setting we call recursive reinforcement learning. A recursive MDP is a collection of MDP components, where each component has the ability to recursive call each other. This allows the introduction of an unbounded stack. We propose abstracting this stack with the costs of the exits of a component. Using this abstraction, we introduce Recursive Q-learning---a model-free reinforcement learning algorithm for RMDPs. We demonstrate the potential of our approach on a set of examples that includes probabilistic program synthesis, a single-exit RMDP, and an MDP composed with a context-free property.

\paragraph{Acknowledgments.}
\includegraphics[height=8pt]{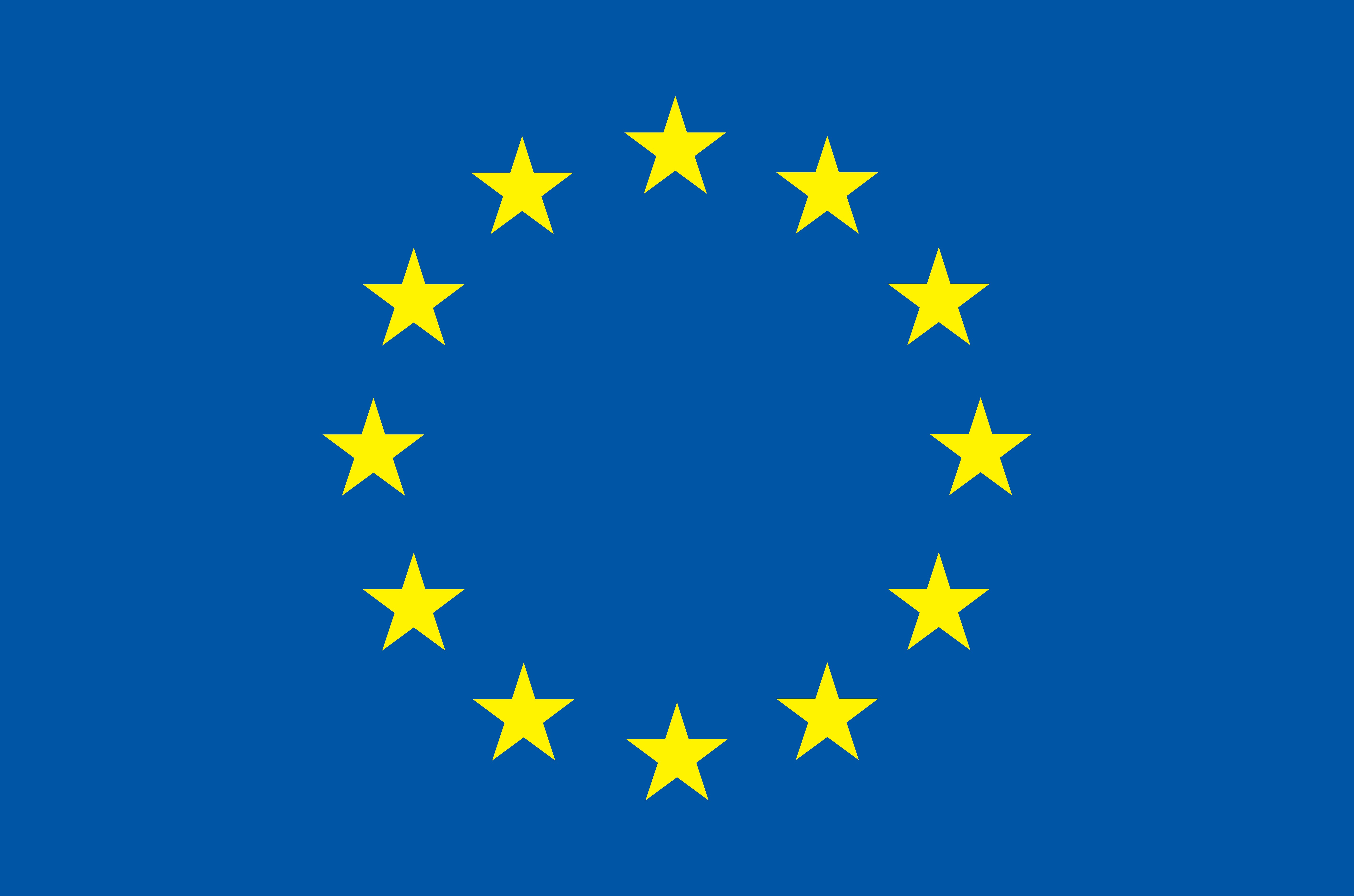} This project has received funding from the European Union’s Horizon 2020 research and innovation programme under grant agreements 864075 (CAESAR), and 956123 (FOCETA).
This work is supported in part by the National Science Foundation (NSF) grant CCF-2009022 and by NSF CAREER award CCF-2146563. 
This work utilized the Summit supercomputer, which is supported by the National Science Foundation (awards ACI-1532235 and ACI-1532236), the University of Colorado Boulder, and Colorado State University. The Summit supercomputer is a joint effort of the University of Colorado Boulder and Colorado State University.

\bibliography{papers}

\begin{thebibliography}{10}

\bibitem{agarwal2019reinforcement}
Alekh Agarwal, Nan Jiang, Sham~M Kakade, and Wen Sun.
\newblock Reinforcement learning: Theory and algorithms.
\newblock {\em CS Dept., UW Seattle, Seattle, WA, USA, Tech. Rep}, 2019.

\bibitem{AKY99}
Rajeev Alur, Sampath Kannan, and Mihalis Yannakakis.
\newblock Communicating hierarchical state machines.
\newblock In Jir{\'i} Wiedermann, Peter van Emde~Boas, and Mogens Nielsen,
  editors, {\em Automata, Languages and Programming}, pages 169--178, Berlin,
  Heidelberg, 1999. Springer Berlin Heidelberg.

\bibitem{ball2000bebop}
Thomas Ball and Sriram~K Rajamani.
\newblock Bebop: A symbolic model checker for boolean programs.
\newblock In {\em International SPIN Workshop on Model Checking of Software},
  pages 113--130. Springer, 2000.

\bibitem{ball2000slam}
Thomas Ball and Sriram~K Rajamani.
\newblock The {SLAM} toolkit.
\newblock In {\em Proceedings of CAV’2001 (13th Conference on Computer Aided
  Verification)}, volume 2102, pages 260--264, 2000.

\bibitem{barto2003recent}
Andrew~G Barto and Sridhar Mahadevan.
\newblock Recent advances in hierarchical reinforcement learning.
\newblock {\em Discrete event dynamic systems}, 13(1-2):41--77, 2003.

\bibitem{bertsekas1991analysis}
Dimitri~P Bertsekas and John~N Tsitsiklis.
\newblock An analysis of stochastic shortest path problems.
\newblock {\em Mathematics of Operations Research}, 16(3):580--595, 1991.

\bibitem{borkar2000ode}
Vivek~S Borkar and Sean~P Meyn.
\newblock The ode method for convergence of stochastic approximation and
  reinforcement learning.
\newblock {\em SIAM Journal on Control and Optimization}, 38(2):447--469, 2000.

\bibitem{Brockm16}
G.~Brockman, V.~Cheung, L.~Pettersson, J.~Schneider, J.~Schulman, J.~Tang, and
  W.~Zaremba.
\newblock {OpenAI Gym}.
\newblock {\em CoRR}, abs/1606.01540, 2016.

\bibitem{Corballis}
Michael~C. Corballis.
\newblock {\em The Recursive Mind: The Origins of Human Language, Thought, and
  Civilization}.
\newblock Princeton University Press, 2014.

\bibitem{cormen2022introduction}
Thomas~H Cormen, Charles~E Leiserson, Ronald~L Rivest, and Clifford Stein.
\newblock {\em Introduction to algorithms}.
\newblock MIT press, 2022.

\bibitem{dietterich2000hierarchical}
Thomas~G Dietterich.
\newblock Hierarchical reinforcement learning with the maxq value function
  decomposition.
\newblock {\em Journal of artificial intelligence research}, 13:227--303, 2000.

\bibitem{esparza2000efficient}
Javier Esparza, David Hansel, Peter Rossmanith, and Stefan Schwoon.
\newblock Efficient algorithms for model checking pushdown systems.
\newblock In {\em International Conference on Computer Aided Verification},
  pages 232--247. Springer, 2000.

\bibitem{EtessamiWY19}
Kousha Etessami, Dominik Wojtczak, and Mihalis Yannakakis.
\newblock Recursive stochastic games with positive rewards.
\newblock {\em Theor. Comput. Sci.}, 777:308--328, 2019.

\bibitem{EtessamiY12}
Kousha Etessami and Mihalis Yannakakis.
\newblock Model checking of recursive probabilistic systems.
\newblock {\em {ACM} Trans. Comput. Log.}, 13(2):12:1--12:40, 2012.

\bibitem{EtessamiY15}
Kousha Etessami and Mihalis Yannakakis.
\newblock Recursive markov decision processes and recursive stochastic games.
\newblock {\em J. {ACM}}, 62(2):11:1--11:69, 2015.

\bibitem{goodfellow2016deep}
Ian Goodfellow, Yoshua Bengio, Aaron Courville, and Yoshua Bengio.
\newblock {\em Deep learning}, volume~1.
\newblock MIT Press, 2016.

\bibitem{griesmayer2006repair}
Andreas Griesmayer, Roderick Bloem, and Byron Cook.
\newblock Repair of boolean programs with an application to c.
\newblock In {\em International Conference on Computer Aided Verification},
  pages 358--371. Springer, 2006.

\bibitem{Hahn21a}
Ernst~Moritz Hahn, Mateo Perez, Sven Schewe, Fabio Somenzi, Ashutosh Trivedi,
  and Dominik Wojtczak.
\newblock Model-free reinforcement learning for branching markov decision
  processes.
\newblock In Alexandra Silva and K.~Rustan~M. Leino, editors, {\em Computer
  Aided Verification - 33rd International Conference, {CAV} 2021, Virtual
  Event, July 20-23, 2021, Proceedings, Part {II}}, volume 12760 of {\em
  Lecture Notes in Computer Science}, pages 651--673. Springer, 2021.

\bibitem{harris1964theory}
Theodore~Edward Harris.
\newblock The theory of branching process.
\newblock Technical Report R-381-PR, Rand Corporation, May 1964.

\bibitem{lari1990estimation}
Karim Lari and Steve~J Young.
\newblock The estimation of stochastic context-free grammars using the
  inside-outside algorithm.
\newblock {\em Computer speech \& language}, 4(1):35--56, 1990.

\bibitem{MADANI20035}
Omid Madani, Steve Hanks, and Anne Condon.
\newblock On the undecidability of probabilistic planning and related
  stochastic optimization problems.
\newblock {\em Artificial Intelligence}, 147(1):5--34, 2003.
\newblock Planning with Uncertainty and Incomplete Information.

\bibitem{manning1999foundations}
Christopher Manning and Hinrich Schutze.
\newblock {\em Foundations of statistical natural language processing}.
\newblock MIT press, 1999.

\bibitem{meiss2007differential}
James~D Meiss.
\newblock {\em Differential dynamical systems}.
\newblock SIAM, 2007.

\bibitem{Mnih15}
V.~Mnih, K.~Kavukcouglu, D.~Silver, A.~A. Rusu, J.~Veness, M.~G. Bellemare,
  A.~Graves, M.~Riedmiller, A.~K. Fidjeland, G.~Ostrovski, S.~Petersen,
  C.~Beattie, A.~Sadik, I.~Antonoglou, H.~King, D.~Kumaran, D.~Wierstra,
  S.~Legg, and D.~Hassabis.
\newblock Human-level control through reinforcement learning.
\newblock {\em Nature}, 518:529--533, February 2015.

\bibitem{parr1998reinforcement}
Ronald Parr and Stuart~J Russell.
\newblock Reinforcement learning with hierarchies of machines.
\newblock In {\em Advances in neural information processing systems}, pages
  1043--1049, 1998.

\bibitem{Put94}
M.~L. Puterman.
\newblock {\em Markov Decision Processes: Discrete Stochastic Dynamic
  Programming}.
\newblock John Wiley \& Sons, Inc., New York, NY, USA, 1994.

\bibitem{Silver16}
D.~Silver, A.~Huang, C.~J. Maddison, A.~Guez, L.~Sifre, G.~van~den Driessche,
  J.~Schrittwieser, I.~Antonoglou, V.~Panneershelvam, M.~Lanctot, S.~Dieleman,
  D.~Grewe, J.~Nham, N.~Kalchbrenner, I.~Sutskever, T.~Lillicrap, , M.~Leach,
  K.~Kavukcuoglu, , T.~Graepel, and D.~Hassabis.
\newblock Mastering the game of {Go} with deep neural networks and tree search.
\newblock {\em Nature}, 529:484--489, January 2016.

\bibitem{skolem1923begrundung}
Thoralf Skolem.
\newblock {\em Begr{\"u}ndung der elementaren Arithmetik durch die
  rekurrierende Denkweise ohne Anwendung scheinbarer Ver{\"a}nderlichen mit
  unendlichem Ausdehnungsbereich}.
\newblock Dybusach, 1923.

\bibitem{soare2016turing}
Robert~I Soare.
\newblock {\em Turing computability: Theory and applications}.
\newblock Springer, 2016.

\bibitem{Sutton18}
R.~S. Sutton and A.~G. Barto.
\newblock {\em Reinforcement Learnging: An Introduction}.
\newblock MIT Press, second edition, 2018.

\bibitem{sutton1999between}
Richard~S Sutton, Doina Precup, and Satinder Singh.
\newblock Between mdps and semi-mdps: A framework for temporal abstraction in
  reinforcement learning.
\newblock {\em Artificial intelligence}, 112(1-2):181--211, 1999.

\bibitem{vatutin1993branching}
Vladimir~A Vatutin and Andre{\i}~M Zubkov.
\newblock Branching processes. ii.
\newblock {\em Journal of Soviet Mathematics}, 67(6):3407--3485, 1993.

\bibitem{vinyals2019grandmaster}
Oriol Vinyals, Igor Babuschkin, Wojciech~M Czarnecki, Micha{\"e}l Mathieu,
  Andrew Dudzik, Junyoung Chung, David~H Choi, Richard Powell, Timo Ewalds, and
  Petko Georgiev.
\newblock Grandmaster level in starcraft ii using multi-agent reinforcement
  learning.
\newblock {\em Nature}, 575(7782):350--354, 2019.

\bibitem{watkins1992q}
Christopher~JCH Watkins and Peter Dayan.
\newblock Q-learning.
\newblock {\em Machine learning}, 8(3):279--292, 1992.

\bibitem{watkins1989learning}
Christopher John Cornish~Hellaby Watkins.
\newblock {\em Learning from delayed rewards}.
\newblock PhD thesis, King's College, Cambridge United Kingdom, 1989.

\end{thebibliography}
\newpage

\appendix

\section{Implementation details}

Algorithm~\ref{alg:deep-rql} shows the extension of tabular Recursive Q-learning to the neural network setting by using the additional techniques introduced by DQN. The algorithm performs two steps per sample: the first computes $v$ and pushes the required values to the replay buffer, and the second samples the replay buffer and updates the neural network via the mean squared error. 

For sampling trajectories, we sampled actions with the standard $\varepsilon$-greedy policy. We matched the hyperparameters for (deep) Recursive Q-learning and (deep) Q-learning on each example, with the discount factor set to $\lambda = 1$ for the latter. The experiments were run on a server with 12 CPU cores and no GPU.

\paragraph{Hyperparameters.}
\phantom{a}

\begin{figure*}[h]
\begin{minipage}[t]{0.33\textwidth}
\centering

Cloud
\vspace{0.1in}

\begin{tabular}{c|c}
    Parameter & Value \\ \hline
    Test samples & $100$ \\
    Learning rate $\alpha$ & $0.02$ \\
    Exploration rate $\varepsilon$ & $0.1$ \\
    Quantize amount & $0.001$
\end{tabular}
\end{minipage}%
\begin{minipage}[t]{0.33\textwidth}
\centering

{Spelunking}
\vspace{0.07in}

\begin{tabular}{c|c}
    Parameter & Value \\ \hline
    Test samples & $100$ \\
    Learning rate $\alpha$ & $0.2$ \\
    Exploration rate $\varepsilon$ & $0.1$
\end{tabular}
\end{minipage}%
\begin{minipage}[t]{0.33\textwidth}
\centering

Palindrome
\vspace{0.1in}

\begin{tabular}{c|c}
    Parameter & Value \\ \hline
    Test samples & $100$ \\
    Initial $\varepsilon$ & $1$ \\
    Final $\varepsilon$ & $0.1$ \\
    Final $\varepsilon$ timestep & $30k$ \\
    Buffer size & $20k$ \\
    Buffer warm up & $1k$ \\
    Batch size $N$ & $256$ \\
    Update frequency $C$ & 500 \\
    Hidden layers & $2$ \\
    Hidden dimension & $128$ \\
    Activation function & tanh \\ 
    Learning rate & $0.0005$ \\
    Optimizer & Adam
\end{tabular}
\end{minipage}
\end{figure*}

\begin{algorithm}[t]
\caption{Deep Recursive Q-learning (DQN-style)}\label{alg:deep-rql}

Buffer size $N$, update frequency $C$

Initialize network parameters $\theta$

Set target network parameters $\theta^{-} \gets \theta$

Initialize empty replay buffer

\While{not converged}{
    $v \gets \mathbf{0}$

    $\text{stack} \gets \emptyset$

    Sample trajectory $\tau \sim \{(s,a,r,s'), ...\}$
    
    \For{$s,a,r,s'$ in $\tau$}{
    
    \tcp*[h]{Push update to replay buffer}
    
    \uIf{entered box}{
    
    $\{s_{\text{exit}_1}, \ldots, s_{\text{exit}_n}\} \gets \text{getExits}(s')$ 
    
    $v' \gets [\max_{a' \in A(s_{\text{exit}_1})} Q(s_{\text{exit}_1}, v, a'; \theta), \ldots, \max_{a' \in A(s_{\text{exit}_n})} Q(s_{\text{exit}_n}, v, a'; \theta)]$
    
    $v'_{\min{}} \gets \min(v')$
    
    $v' \gets v' - v'_{\min{}}$
    
    buffer.add(\emph{entered box}, ($s$, $v$, $a$, $r$, $s'$, $v'$), $v'_{\min{}}$)
    
    stack.push($v$)
    
    $v \gets v'$
    }
    \uElseIf{exited box}{
    $\{s_{\text{exit}_1}, \ldots, s_{\text{exit}_n}\} \gets \text{getExits}(s)$
    
    Set $k$ such that $s' = s_{\text{exit}_k}$
    
    buffer.add(\emph{exited box}, ($s$, $v$, $a$, $r$, $s'$, $v$), $v(k)$)
    
    $v \gets$ stack.pop()
    }
    \Else{
        buffer.add(\emph{normal}, ($s$, $v$ $a$, $r$, $s'$, $v$), $\bot$)
    }
    
    \vspace{0.1in}
    
    \tcp*[h]{Update network}
        
    Sample minibatch $\{\emph{type}_j, (s_j, v_j, a_j, r_j, s'_j, v'_j), \emph{aux}_j\}_{j=1}^N$ of size $N$ from replay buffer
    
    \For{$j = 1,\ldots,N$}
    {$
    \text{targ}_j \gets 
    \begin{cases}
        r_j + \max_{a' \in A(s'_j)} Q(s'_j, v'_j, a'_j; \theta^{-}) + \emph{aux}_j & \emph{type}_j = \emph{entered box} \\
        r_j + \emph{aux}_j & \emph{type}_j = \emph{exited box}\\
        r_j +  \max_{a \in A(s'_j)} Q(s'_j, v'_j, a'_j; \theta^{-}) & \emph{type}_j = \emph{normal}
    \end{cases}
    $}
    
    $\mathcal{L} = \frac{1}{N} \sum_{j=1}^N (Q(s_j,v_j,a_j) - \text{targ}_j)^2$
    
    Update $\theta$ with respect to loss $\mathcal{L}$ with gradient descent
    
    Set $\theta^{-} \gets \theta$ every $C$ steps
    }
}

\Return $\theta$
\end{algorithm}

\section{Discussion on discounting}
There are multiple choices for discounting in RMDPs. The most straightforward choice is that discounting by $\lambda$ is equivalent to stopping the entire process with probability $1-\lambda$. We call this step-wise discounting. One can transform this type of discounting into a total reward model satisfying Assumption~\ref{asm:proper} by adding a special exit (the \emph{exit-lane}) which leads to the special exit in the box above with no reward. This type of discounting is further discussed in Appendix~\ref{sec:PAC}. An alternative choice for discounting in single-exit RMDPs is that discounting by $\lambda$ is equivalent to stopping the current box---by leaving out of its only exit---with probability $1-\lambda$. 
We call this box-wise discounting. 
With this type of discounting, the discount factor can be incorporated into Recursive Q-learning by simply multiplying the terms $\max_{a' \in A(S')} Q(s', a')$ and $\max_{a' \in A(s_{\text{exit}})} Q(s_{\text{exit}}, a')$ in Algorithm~\ref{alg:rql1e} by $\lambda$. However, box-wise discounting does not necessarily ensure properness for all single-exit RMDPs and discount factors $\lambda < 1$. Instead, the discount factor must be sufficiently small. Note that box-wise and step-wise discounting schemes are equivalent in MDPs (RMDPs with no recursive calls). Our results subsume these settings by considering total reward under the properness assumption.

\section{Exponential Succinctness of Hierchical MDPs (Proof from Section~\ref{sec:intro})}

\begin{lemma}
\label{lem:exp_succ}
Hierarchical MDPs are exponentially more succinct than finite-state MDPs.
\end{lemma}
\begin{proof}
This proof is adapted from a similar result on the (non-stochastic) recursive state machines from~\cite{AKY99}.
Consider a collection of hierarchical, deterministic MDPs $M_1, M_2, \ldots, M_n$. 
The MDP $M_1$ upon taking an action $a$ gives a reward of $1$ and terminates, while it gives a reward of $0$ for any other action and moves to a sink. 
Each MDP $M_i$ (for $i > 1$) upon an action $a$ calls the MDP $M_{i-1}$ twice in sequence and then accepts with a reward of $1$, and for other actions it makes a transition to a sink without providing any reward.
This MDP $M_n$ has the property that the optimal value is $2^n-1$ and the optimal policy corresponds to choosing $2^n-1$ $a$'s in succession. 
This environment can only be expressed by a finite-state MDP with at least $2^n-1$ states.
\end{proof}

\section{Proof of Theorem~\ref{thm:totalundec}}
\undec*
\begin{proof}
We make use of the existing results for the {\em probabilistic finite automata} (PFAs) model. PFAs are essentially finite automata where nondeterminism is replaced by probabilistic transitions. Specially, when a letter is read, the next state is selected by chance with a fixed probability distribution (that depends on the current state and letter only) instead of the controller selecting the new state. A word is accepted by a PFA if the probability of reaching the special accept state after reading this word is higher than a given fixed threshold $\lambda \in [0,1]$. Madani, Hanks, and Condon~\cite{MADANI20035} showed undecidability of checking the non-emptiness for $\lambda = 1/2$ and a {\em leaky PFA} that is a PFA in which at each step we stop the run (or, equivalently, move to a non-accepting state with a self loop) with a fixed probability $\gamma$. 

As shown in~\cite{EtessamiWY19}, RMDPs can simulate probabilistic finite automata. We adapt this to show how proper RMDPs can simulate leaky PFAs. Such a RMDP has a single component consisting of a box corresponding to each input letter. The number of exits of this component is the number of states in the PFA plus one additional exit-lane exit. One of these exits corresponds to the special accepting state of the PFA. There is a single entry of this component which is the only non-trivial choice point for the controller. Each input letter has a corresponding action that leads to the box corresponding to this letter with probability $1-\gamma$ and with probability $\gamma$ to the exit-lane exit. There is a special start action that leads to the exit corresponding to the initial state of the PFA. Once a box corresponding to a letter, $a$, is exited the probabilistic transitions corresponding to the effect of reading $a$ in the PFA takes place by transitioning to the exit of the component corresponding to the new state of the simulated PFA. An exit-lane exit of any box is connected directly to the exit-lane exit of the component it is in and the reward for such transitions is 0. Note that all that that happens once exit-lane is reached is that the whole content of the stack is popped which results in terminating without modifying the accumulated reward so far. Note that this RMDP is proper, because in each step there is a fixed $\gamma$ chance of entering the exit-lane and terminating (once the whole stack is popped), so the expected number of steps taken by any strategy is finite. 

We claim that the controller has a strategy that terminates at the accepting exit (with empty stack content) with probability $\geq \frac{1}{2}$ iff there exists a word accepted by a leaky PFA (and so its language is non-empty) with the same probability $\geq \frac{1}{2}$. The strategy would pick the letter of the input word in reverse by calling the corresponding box and once done select the special start action. It is easy to see that the behavior of the leaky PFA (including stopping with probability $\gamma$) is mimicked precisely.

Now, one can easily encode termination objective using rewards: assign reward of 1 to the final single transition just before the RMDP reaches the accepting exit with an empty stack content and reward of 0 to all other transitions. This shows that the the strategy existence problem for the expected total reward objectives is at least as hard as the strategy existence problem for termination objective which we showed to be as hard as non-emptiness of the language accepted by a leaky PFA; an undecidable problem.
\end{proof}

Note that the above proof carries over to a discounted setting, so the strategy existence problem in such a setting is also undecidable.  

\section{Proof of Theorem \ref{thm:fix-point-correspondence}}
\fixedpc*

\begin{proof}
We first show that the fixed point $y$ of $\textrm{OPT}_{\tt recur}$ is unique and 
$y(\sseq{\kappa}, q) = \ETotal^M(\sseq{\kappa}, q)$.

The values $\ETotal^M(\sseq{\kappa}, q)$ are indeed a fixed point of $\textrm{OPT}_{\tt recur}$. Consider any state $\sseq{\kappa}, q$. We proceed over all possible cases for the type of $q$.

If $q {=} (b, en) \in \Call$ then
$\ETotal^M(\sseq{\kappa}, q) = \ETotal^M(\sseq{\kappa, b}, en)$, and if
$q \in \Ex, (b, q) \in \return(b), \kappa {=} \sseq{\kappa', b}$ then
$\ETotal^M(\sseq{\kappa}, q) = \ETotal^M(\sseq{\kappa'}, (b, q))$, both 
by definition of MDP $\sem{M}$. 
Similarly, if $q \in \Ex$ then
$\ETotal^M(\sseq{\emptyset}, q) = 0$ as we immediately terminate.
Finally, if $q$ is any other type of vertex, then 
\[
\ETotal^M(\sseq{\kappa}, q) = \max\limits_{a \in A(q)} \Bigl\{ r(q, a) {+} \sum\limits_{q' \in Q}  p(q' | q, a) \ETotal^M(\sseq{\kappa}, q')\Bigr\},
\]
because the best one can do while at $(\sseq{\kappa}, q)$ is to pick an action that maximizes the one-step reward plus the weighted average of the expected total reward from the successor state. 

On the other hand, any strategy that picks any such an action achieves $\ETotal^M(\sseq{\kappa}, q)$. Let us denote such a strategy by $\sigma$. This works because $\sigma$ is guaranteed to terminate within a finite number of steps, $K$, when starting at any $(\langle\emptyset\rangle,q)$ due to Assumption \ref{asm:proper}. Note that the time to exit the current box when at any $(\langle\kappa\rangle,q)$ is also at most $K$. This is because when starting at $(\langle\kappa\rangle,q)$ and at $(\langle\emptyset\rangle,q)$ the transition structure of the model looks the same until an exit of the current box is reached. (Specifically, a mapping of $(\langle\kappa \kappa'\rangle,q)$ to $(\langle\kappa'\rangle,q)$ for every $\kappa'$ and $q$ is an isomorphism.) As a result, we get an upper bound of $K\cdot|\kappa|$ for the expected time to terminate from $(\langle\kappa\rangle,q)$ for any $\kappa$ and $q$.  

Note now that, for any $\epsilon >0$, the probability of terminating when starting at $(\langle\kappa\rangle,q)$ after $K\cdot |\kappa| / \epsilon$ steps is at most $\epsilon$, because otherwise the expected termination time would not $\leq K\cdot|\kappa|$. 
This gives us a bound of $K\cdot|\kappa|\cdot r_{max}$ for the expected total reward from $(\langle\kappa\rangle,q)$, because in each transition we can get at most $r_{max}$.
In such a setting, the results from Chapter 7 of \cite{Put94} imply that $\sigma$ is optimal.

If there was another fixed point such that at least one coordinate is higher than $\ETotal^M$, then a  strategy constructed as above would obtain more than the supremum over all possible strategies of the expected total reward when starting at that state; a contradiction. On the other hand, if there was any other smaller fixed point than $\ETotal^M$, then by using repeatedly using choices made by the strategy $\sigma$ defined above we would converge at $\ETotal^M$, but on the other we should never be able to improve based on the assumption that this is a fixed point; a contradiction. (For more details see the end of the proof of Theorem \ref{thm:unique-fixed-point} in Appendix \ref{sec:single}).

Now suppose that there exists a fixed point point $y$ of $\textrm{OPT}_{\tt recur}$ and 
a fixed point $x$ of $\textrm{OPT}_{\tt cont}$ such that
$y(\sseq{\emptyset},q) \neq x({\bf 0},q)$.
We now show how to reconstruct a fixed point $y'(\sseq{\kappa}, q)$ of $\textrm{OPT}_{\tt recur}$ that corresponds exactly to the fixed point $x({\bf v}, q)$. This would lead to a contradiction as we just showed that $\textrm{OPT}_{\tt recur}$ has a unique fixed point. 

This will be done by a recursive parallel fixed point reconstruction process below. Let the current call stack be $\sseq{\kappa}$ and valuation of the exits be $\bf v$. We start the process at $\kappa = \emptyset$ and $\bf v = \bf 0$.
We assign the values as follows.
\begin{eqnarray*}
y'(\sseq{\kappa}, q) &=& 
\begin{cases}
    \text{make recursive call for $\kappa:=\kappa b$ and ${\bf v}:=(x({\bf v},q'))_{q'\in \return_b}$} & q = (b, en) \in \Call\\
    {\bf v}(q) & q \in \Ex \\
\max\limits_{a \in A(q)} \Bigl\{ r(q, a) {+} \sum\limits_{q' \in Q}  p(q' | q, a) x({\bf v}, q')\Bigr\} & \text{otherwise.}
\end{cases}
\end{eqnarray*}

It is clear that the value of all the states are the same in $y'$ and $x$, so the optimal strategies would be the same as well.  
\end{proof}

\section{PAC Learnability for RMDPs with Discounting and the Milder Restrictions from Section \ref{sec:rmdp}}
\label{sec:PAC}

In this section, we start with defining discounted rewards with the usual semantics. We then provide PAC learning results for RMDPs with discounted rewards (instead of the requirements on the expected decline of the stack size and the expected rewards) (Section \ref{ssec:discounted}),
and then extend these results to the undiscounted case (Section \ref{ssec:undiscounted}).

\subsection{Discounted Rewards}
In the discounted setting, the objective in a RMDP $M$ is to find a strategy $\sigma \in \Sigma_M$ that maximizes the {\it discounted reward} 
$\EDisct(\lambda)^M_\sigma(s)$, which is defined as
\[
\EDisct(\lambda)^M_\sigma(s) = \lim_{N \to \infty} \eE^M_\sigma(s) 
\left \{\sum_{1 \leq i \leq N}
  \lambda^{i-1} r(X_{i-1}, Y_i)\right \},
\]
for some discount factor $0 \leq \lambda < 1$. 
The optimal (discounted) value $\EDisct(\lambda)^M_*(s)$ is then defined as 
\[
\EDisct(\lambda)^M_*(s) = \sup_{\sigma \in \Sigma_M} \EDisct(\lambda)^M_\sigma(s).
\]
We say that $\sigma$ is discounted optimal if $\EDisct(\lambda)^M_*(s) = \EDisct(\lambda)^M_\sigma(s)$ for every $s \in S$.

\subsection{PAC Learning for RMDPs with Discounted Rewards}
\label{ssec:discounted}

\begin{lemma}\label{lemma:RMDP}
Let $M$ be an RMDP with diameter $d$, $M'$ be an RMDP that differs from $M$ only by using a different transition function $\delta_{M'}$ $\varepsilon$-close to $\delta_M$, and let $\sigma$ an $\varepsilon$-proper strategy.
Then $|\EDisct(\lambda)^M_\sigma(s)-\EDisct(\lambda)^{M'}_\sigma(s)| \leq \frac{\varepsilon d}{(1-\lambda)^2}$.
\end{lemma}

\begin{proof}
We can simply estimate
\newline
$$\begin{array}{ll}
|\EDisct(\lambda)^M_\sigma(s)-\EDisct(\lambda)^{M'}_\sigma(s)| & \leq \sum_{i = 1}^\infty
  \lambda^{i-1}| \eE^M_\sigma(s)r(X_{i-1}, Y_i) - 
  \eE^{M'}_\sigma(s)r(X_{i-1}, Y_i) 
  |\\
  &\leq \sum_{i = 1}^\infty d \lambda^{i-1}(1-(1-\varepsilon)^{i}) = \frac{d}{1-\lambda}-\frac{(1-\varepsilon) d}{1-\lambda + \lambda\varepsilon}
  \\ &
  = \frac{\varepsilon d}{(1-\lambda)(1-\lambda+\lambda\varepsilon)} \leq \frac{\varepsilon d}{(1-\lambda)^2}\ ,
\end{array}$$
where the first inequality is by triangulation, while the second is estimating the chance, that the difference between $\delta$ and $\delta'$ has been realized within the first $i$ steps by $1-(1-\varepsilon)^i$.
This bounds the sum of the different probabilities of taking a transition; the difference in the rewards is in this case bounded by the diameter $d$. 
\end{proof}

As the estimation from the previous lemma survives a supremum operation over all $\varepsilon$-proper strategies we get the following corollary.

\begin{corollary}
For an $\varepsilon'$-proper RMDP $M$ with diameter $d$, and an $M'$ that differs from $M$ only by using a different transition function $\delta_{M'}$ $\varepsilon$-close to $\delta_M$ for $\varepsilon \leq \varepsilon'$, then $|\EDisct(\lambda)^M(s)-\EDisct(\lambda)^{M'}(s)| \leq \frac{\varepsilon d}{(1-\lambda)^2}$.
\end{corollary}

\begin{corollary}
Consequently, it suffices to learn the transition function $\delta_M$ of an $\varepsilon'$-proper RMDP
\begin{itemize}
    \item with precision $\varepsilon'' = \min\big\{\varepsilon',\frac{\varepsilon(1-\lambda)^2}{d}\big\}$ to obtain (efficient) PAC learnability if we can (efficiently) evaluate the approximate RMDP $M'$ obtained, or
    \item with precision $\varepsilon'' = \frac{1}{2}\min\big\{\varepsilon',\frac{\varepsilon(1-\lambda)^2}{d}\big\}$ if we can (efficiently) approximately evaluate $M'$ with precision $\varepsilon''$
\end{itemize} 
with probability $\delta$.
\end{corollary}

But this is a standard condition that does not differ much from MDPs and provides the following theorem.

\begin{theorem}
Taking as input $\frac{1}{\varepsilon'},\frac{1}{\varepsilon},\frac{1}{\delta},d,n$, we can construct with probability $\geq \delta$, for an $\varepsilon'$-proper RMDP M with $n$ states and diameter $d$, $\EDisct(\lambda)^M(s)$ with precision $\varepsilon$.
If the evaluation (approximation) of an RMDP $M'$ that differs from $M$ only in using $\delta_{M'}$ instead of $\delta_M$ with precision $\min\big\{\varepsilon',\frac{\varepsilon(1-\lambda)^2}{d}\big\}$ ($\frac{1}{2}\min\big\{\varepsilon',\frac{\varepsilon(1-\lambda)^2}{d}\big\}$)
is tractable in the parameters above, then we can efficiently PAC learn $\frac{1}{2}\min\big\{\varepsilon',\frac{\varepsilon(1-\lambda)^2}{d}\big\}$.
\end{theorem}

Naturally, the discounted value for a given RMDP can be approximated by unraveling sufficiently deeply.

\begin{corollary}
For every $\varepsilon'$-proper RMDP, $\EDisct(\lambda)^M(s)$ is PAC-learnable.
\end{corollary}

\subsection{PAC Learning for RMDPs with Undiscounted Rewards}
\label{ssec:undiscounted}

We now generalize this result to the undiscounted case we refer to in the paper.
The first step is a reduction, which shows that the discounted case can be viewed as a special case of the undiscounted case by encoding the discounted payoffs by using an adjusted RMDP.
We then continue to adjust the resulting RMDP further to add costs to ``stopping'' the RMDP in accordance with the level we are in.

This will then turn into a cost estimation, whereby the exit lane is entered to overestimate the effect of changing the strategy.
For this to be estimated with a low cost, we use the knowledge that the expected number of steps till termination is finite.

\subsubsection{From Discounting to Stopping}
\label{sssec:exitLane}
In MDPs, discounting can be simulated by stopping in each step with a chance of $1-\lambda$, while continuing with a chance of $\lambda$.
In RMDPs, there is no equivalent of stopping immediately, as we will first have to move down the stack.
However, this can be simulated by adding a special exit to each component and block: from this exit of a block, we always continue with a reward of $0$ to the special exit of the component, and thus have a cost-free direct way to termination, albeit it takes as many steps as the stack is high.

Once equipped with such an \emph{exit lane} without rewards, we can use a reduction similar to the standard stopping reduction for MDPs: we simply go with probability $1-\lambda$ to the special exit of the component, and continue with a probability of $\lambda$, i.e., all former probabilities are multiplied by $\lambda$.

\subsubsection{Adding Rewards to the Exit Lane}

In order to connect the construction with the parameters from Section \ref{sec:rmdp}, we first consider the effect of changing the rewards on the exit lane.

If we change the rewards of all transitions that lead to the special exit from $0$ to an upper bound $b \geq 0$ or a lower bound $-b \leq 0$ for the maximal or minimal expected reward for \emph{any} strategy for \emph{any} state with empty stack for \emph{any} RMDP $M'$ $\varepsilon$-close to $M$, then the collected reward when moving down the \emph{exit lane} with rewards can serve as an over- or underrepresentation of the remaining reward for any strategy.

In turn the over- and underestimations for a stack of hight $h$ are over- and underestimations for the remaining rewards after $h$ steps.

\subsubsection{Connecting our Parameters with Discounting}

To see that the parameters we have provided generalize the case of discounting, note that, for an RMDP with discounting, one can choose $b= \frac{d}{1-\lambda}$, $c_o = 1+ \frac{1}{1-\lambda}$, and $\mu = 1 -\lambda$.

To justify the choice of $b= \frac{d}{1-\lambda}$, we observe that the expected sum of the weight-discounted number of steps is at most $\frac{1}{1-\lambda}$, such that the absolute value of the rewards collected is at most the diameter times this number.

To justify the choice for $c_o$ and $\mu$, we first define and estimate a different value for $M^+$ (and $M^-$): we define the value of a state as
\begin{enumerate}
    \item the height of the stack plus $\frac{1}{1-\lambda}$ if the RMDP is still running and we are not on the exit lane,
    \item the height of the stack if we are on the exit lane, and
    \item $0$ if the RMDP has terminated.
\end{enumerate}
In the first case, the expected value of the state will drop by at least $1-\lambda$: this is because it can only increase by $1$ when the exit lane is not entered (which happens with a likelihood of $\lambda$), and goes down by $1+\frac{1}{1-\lambda}$ if the exit lane is entered (with a likelihood of $1-\lambda$), which leads to an expected value that is at least $2(1-\lambda)$ smaller.

In the second case, the value is always reduced by $1$.

Thus, whenever the RMDP has not terminated after $k$ steps, the expected value of this sum is reduced by (at least) $\min\{1,2(1-\lambda)$\}, and therefore by at least $1-\lambda$.

Consequently, the expected value of this sum is bounded by $c_o - \mu \cdot \sum_{i=1}^k p_{\mathsf{run}}^{M_\sigma'}(k)$, and it in turn bounds the stack height.

\subsubsection{Wrapping up the Proof}

\totalPAC*

To wrap this up to show Theorem \ref{theo:totalPAC}, we can estimate the effect of moving to a different position due to a change of a strategy by moving to the exit lane with either reward $b$ (for overestimation) or $-b$ (for underestimation).

Thus, if we refer to an (arbitrary but fixed) RMDP $M'$ $\varepsilon$-close to $M$, and to the over- and under approximation of with exit lanes with rewards $b$ and $-b$ and a likelihood to transition to the exit lane of $\varepsilon$ as $M^+$ and $M^-$, respectively, we have 

$$\ETotal_\sigma^{M^+}\big((\langle\emptyset\rangle,q)\big) \geq
\ETotal_\sigma^{M'}\big((\langle\emptyset\rangle,q)\big)  \geq
\ETotal_\sigma^{M^-}\big((\langle\emptyset\rangle,q)\big) $$

for all strategies $\sigma$, and thus

$$\ETotal^{M^+}\big((\langle\emptyset\rangle,q)\big)  \geq
\ETotal^{M'}\big((\langle\emptyset\rangle,q)\big)  \geq
\ETotal^{M^-}\big((\langle\emptyset\rangle,q)\big) \ .$$

In order to estimate $\ETotal_\sigma^{M^+}(s) -
\ETotal_\sigma^{M^-}(s)$, we first observe that the expected stack height must be converging to $0$ for $M_\sigma$ by our assumption that it is bounded by $c_o - \mu \sum_{i-1}^k p_{\mathsf{run}}^{M_\sigma}(k)$ by our assumption.
(Note that, for convergence, it is enough to assume termination in expected finite time.
The stronger assumption is used for to establish PAC learnability.)

For a given $\varepsilon'>0$ we therefore pick a $k$ such that, after $k' \geq k$ steps, the expected stack height is lower than $ \frac{\varepsilon'}{b}$.
E.g., we can choose $k$ such that $k \mu \frac{\varepsilon'}{b} \geq c_o$ holds, such that the falling sequence $p_{\mathsf{run}}^{M_\sigma}(k)$ cannot be above $\frac{\varepsilon'}{b}$ for $k$ steps.
This is satisfied for $k = \Big\lceil \frac{c_ob}{\mu \varepsilon'}\Big\rceil $.

We then pick $\varepsilon = \frac{\varepsilon'}{k^2 b}$. 

In this case, the likelihood of moving onto the exit lane before step $k$ is lower than $\frac{\varepsilon'}{k c}$, and the difference in rewards along the exit lane in this case below $k c$, leaving a contribution to the difference of below $\varepsilon'$.

When entering the exit lane on or after step $k$, the expected stack height is below $ \frac{\varepsilon'}{b}$, leading to a contribution to the expected distance between the rewards for $M^+$ and $M^-$ below $\varepsilon'$.

Together, this provides
$$\ETotal_{\sigma}^{M^+}\big((\langle\emptyset\rangle,q)\big) -
\ETotal_{\sigma}^{M^-}\big((\langle\emptyset\rangle,q)\big)
\leq 2\varepsilon'\ .$$

This shows that the expected value converges. We still need to show that we can approximate it to up to $4\varepsilon'$ with a likelihood of at least $1-\delta$.

For this, we can estimate $\ETotal_{\sigma}^{M^+}\big((\langle\emptyset\rangle,q)\big)$ and $\ETotal_{\sigma}^{M^-}\big((\langle\emptyset\rangle,q)\big)$ by cutting the runs after step $k$, then the error in the expected value is at most $\varepsilon'$.

Estimating $\ETotal_{\sigma}^{M}\big((\langle\emptyset\rangle,q)\big)$ by cutting the runs after $k$ steps leads to a value in between.
Using triangulation, this entails that it estimates the value of $\ETotal_{\sigma}^{M}\big((\langle\emptyset\rangle,q)\big)$.

But this can be unravelled into a finite MDP, and hence be PAC learned up to precision $\varepsilon'$ with likelihood at least $1-\delta$ with standard techniques.

\section{Reinforcement Learning for Single-Exit RMDPs}
\label{sec:single}

\uniqfp*
\begin{proof}[Proof of Theorem \ref{thm:unique-fixed-point}]
The proof proceeds over several lemmas.
\begin{lemma}
$\ETotal^M(q)$ is a fixed point of $F$. 
\end{lemma}
\begin{proof}
Let us consider any vertex $v$. 

If $v$ is an entry port, i.e., 
$q = (b, en) \in \Call, ex = \Ex_{Y(b)}$ then once the box $b$ is entered, by Assumption \ref{asm:proper}, we will be almost surely reach the unique exit $ex'$ of $b$ no matter which action will be picked inside of $b$. Therefore, to maximize the total reward it suffices to focus solely first on maximizing the reward while inside $b$ before $ex'$ is reached no matter what the current call stack is.
The maximal reward possible to obtain is equal to $\ETotal^M(en)$ by definition. 
Once the exit is reached, it suffices to maximize the reward from the exit port $(b,ex')$, which is 
at most equal to $\ETotal^M((b,ex'))$. It is easy to see that these two strategies can be combined into a single strategy from $q$ that can get arbitrarily close to the maximum possible value of $\ETotal^M(en) + \ETotal^M((b,ex'))$. So this all together shows that
$\ETotal^M(q) = \ETotal^M(en) + \ETotal^M((b,ex'))$ holds.

If $v$ is any other non-exit vertex, then the best we can do is to pick any action available at $v$ that maximizes expected reward from the successor vertex plus the one step reward for picking this action at $v$. Once the successor node is reached, we can then switch to the best possible strategy from that node. This shows that the value of $\max_{a\in A(q)} \Bigl\{r(q,a) + \sum\limits_{q' \in Q} p(q'|q, a) x(q') \Bigr\}$ is not only achievable but also the most one can expect to get when starting at $v$.
\end{proof}

\begin{lemma}
\label{lem:FP-stackless}
For any fixed point $\bfx$ of $F$, there exists a stackless strategy $\sigma(\bfx)$ such that $\ETotal^M_{\sigma(\bfx)}(q) = \bfx(q)$ for every $q$. 
\end{lemma}
\begin{proof}
This strategy $\sigma(x)$ will simply pick
\[\sigma(\bfx)(\langle\kappa\rangle,q) = \argmax_{a\in A(q)} \Bigl\{r(q,a) + \sum\limits_{q' \in Q} p(q'|q, a) x(q') \Bigr\}\]
from any vertex $q$ and the call stack $\kappa$. 
This works, because as argued in Theorem \ref{thm:fix-point-correspondence} a mapping of $(\langle\kappa\kappa'\rangle,q)$ to $(\langle\kappa'\rangle,q)$ for all $\kappa'$ and $q$ gives us two isomorphic models until the exit of the top box in $\kappa$ is reached and so the optimal strategy for $(\langle\emptyset\rangle,q)$ works for $(\langle\kappa\rangle,q)$ as well.
In other words, the best thing to do at $(\langle\kappa\rangle,q)$ is to try to maximize the reward before exiting the current box which happens with probability $1$ and the same holds for $(\langle\emptyset\rangle,q)$. In order to maximize this reward, the call stack being $\kappa$ or $\emptyset$ makes no difference as the transitions only depend on the current component.  
\end{proof}

Now, we can observe that $\bfx^*(q) = \ETotal^M(q)$ is the largest fixed point of $F$. If there was $\bfx$ such that $\bfx^*(q) < \bfx(q)$ for some $q$ then we could pick $\sigma(\bfx)$ as our strategy and get the reward of $\bfx(q)$ from $q$ due to Lemma \ref{lem:FP-stackless}. However, just by definition, $\bfx^*(q) = \ETotal^M(q) = \sup_\sigma \ETotal^M(q)$ is the largest possible reward obtainable from $q$ as all strategies $\sigma$ are considered; a contradiction.

Note that once a stackless strategy $\sigma$ is fixed then $F(\bfx)$ becomes equal to $A_\sigma \bfx + \bfb_\sigma$, where $\bfb_\sigma$ is the one-step transitions rewards obtained using $\sigma$ and $A_\sigma$ is a transition matrix derived from $\sigma$. Note that all entries of $A_\sigma$ are non-negative.
We will refer to $F$ as $F_\sigma$ in such a case and exploit this fact later in various proofs.
Note that we can interpret $F_\sigma(\bfx)$ as the expected total reward of using $\sigma$ for each vertex once and then obtaining reward $\bfx$.

\begin{lemma}
\label{lem:proper-1}
For any proper 1-exit RMDP and stackless strategy $\sigma$, we have that $\ETotal^M_\sigma = \sum^\infty_{i=1} A^i_\sigma b_\sigma$.
\end{lemma}
\begin{proof}
Notice that 
\begin{align*}
F_{\sigma}(\bfx) &= A_{\sigma} \bfx + b_{\sigma} \\
F^2_{\sigma}(\bfx) &= A_{\sigma} (A_{\sigma} \bfx + b_{\sigma}) + b_{\sigma} = A^2_{\sigma} \bfx + (A_{\sigma} + I) b_\sigma\\
F^3_{\sigma}(\bfx) &= A_{\sigma} F^2_{\sigma}(\bfx) + b_{\sigma} = A^3_{\sigma} \bfx + (A^2_{\sigma} + A_{\sigma} + I) b_\sigma\\
\vdots & \\
F^k_{\sigma}(\bfx) &= A^k_{\sigma} \bfx + \Bigl(\sum^k_{i=0}  A^i_{\sigma}\Bigr) b_{\sigma}.
\end{align*}

Let $\bf 1$ and $\bf 0$ be vectors consisting only of 1s and 0s, respectively.
As in \cite{EtessamiWY19}, we can show that the expected number of steps taken before termination while using $\sigma$ can be computed by iterating $\bfy_{i+1} = A_\sigma(\bfy_i) + {\bf 1}$ when starting at $\bfy_0 = {\bf 0}$. Intuitively, $\bfy_{i+1}$ will correspond to the expected number of steps taken before termination when we assume the call and return from a box to be executed in parallel in a single step. The crucial assumption made in \cite{EtessamiWY19} that all rewards are positive (and equal to 1) in true in this case. 

This shows that $\lim_{k\to\infty} F^k_{\sigma}({\bf 0}) = A^k_{\sigma} {\bf 0} + \Bigl(\sum^k_{i=0}  A^i_{\sigma}\Bigr) {\bf 1}$ is finite.
Therefore, we have that $\lim_{i\to\infty} A^i_{\sigma}$ is an 0 matrix, because otherwise that sum $\sum^k_{i=0}  A^i_{\sigma}$ would not be finite as a sum of all non-negative matrices.

Now, one can see that $\lim_{k\to\infty} F^k_{\sigma}(\bfx) = \lim_{k\to\infty} A^k_{\sigma} \bfx + \Bigl(\sum^k_{i=0} A^i_{\sigma}\Bigr) b_{\sigma} = \lim_{k\to\infty} \Bigl(\sum^k_{i=0} A^i_{\sigma}\Bigr) b_{\sigma}$ converges absolutely, because $\Bigl(\sum^k_{i=0}  A^i_{\sigma}\Bigr) |b_{\sigma}| \leq \Bigl(\sum^k_{i=0}  A^i_{\sigma}\Bigr) r_{max} \leq K r_{max}$, where $K$ is the expected number of steps taken by any strategy as guaranteed by Assumption \ref{asm:proper}.
It is clear that this limit $\lim_{k\to\infty} F^k_{\sigma}(\bfx)$ is a fixed point of $F_\sigma$, because $F_\sigma(\lim_{k\to\infty} F^k_{\sigma}(\bfx)) = \lim_{k\to\infty} F^{k+1}_{\sigma}(\bfx) = \lim_{k\to\infty} F^k_{\sigma}(\bfx)$. As we already showed $\ETotal^M_\sigma$ is a fixed point of  $F_\sigma$. It is clear that $F_\sigma$ cannot have any other fixed point because, for arbitrary $\bfx$,
we have $\lim_{k\to\infty} F^k_{\sigma}(\bfx) = A^i_{\sigma} b_{\sigma}$ which is fixed in terms of $\bfx$. This shows that $\ETotal^M_\sigma = \sum^\infty_{i=1} A^i_\sigma b_\sigma$ has to hold.
\end{proof}

We are ready to show that no other than $\bfx^* = \ETotal^M$ fixed point of $F$ exists. Suppose there is one and let us denote it by $\bfx$. As we just showed $\bfx \leq \bfx^*$ and $\bfx(q) < \bfx^*(q)$ for some $q$ (because otherwise $\bfx^* = \bfx$). 

Let us denote $\sigma(\bfx^*)$ by $\sigma*$ and consider $F_{\sigma^*}(\bfx)$. If any coordinates of $F_{\sigma^*}(\bfx)$ is larger than $\bfx$ then we can improve the expected total reward by using ${\sigma^*}$ once and then follow the strategy that gives us expected total reward of $\bfx$. A contradiction with the assumption that $\bfx$ is a fixed point of $F$.

So we have $F_{\sigma^*}(\bfx) \leq \bfx$ and by iterating this we get that $F^k_{\sigma^*}(\bfx) \leq \bfx$ for every $k$. In the previous Lemma we showed that $\ETotal^M = \lim_{k\to\infty} F^k_{\sigma^*}(\bfx) \leq \bfx < \ETotal^M$; a contraction. 
\end{proof}

\oneexitPAC*
\begin{proof}
Let us denote the $\epsilon$-proper $1$-exit RMDP by $M$ and its optimal total reward stackless strategy by $\sigma^*$. As we know that every strategy is $\epsilon$-proper then so is $\sigma^*$. This means that for all $M'$ that that are $\varepsilon$-close to $\Mm$ 
the expected time to terminate when using $\sigma^*$ is also $\leq K$.
Recall that $\varepsilon$-closeness means that 
$\sum_{q\in S, a \in A, r\in S}|\delta_\Mm(q,a)(r) - \delta_{\Mm'}(q,a)(r)| \leq \varepsilon$, and where the support of $\delta_{M'}(q,a)$ is a subset of the support of $\delta_{M}(q,a)$ for all $q \in S$ and $a \in A$.

\begin{lemma}
\label{lem:proper-bound}
If stackless $\sigma$ is $\varepsilon$-proper, then $\|\ETotal^M_\sigma - \ETotal^{M'}_\sigma\|_\infty \leq 2 \epsilon K^2 r_{\max}$.
\end{lemma}
\begin{proof}
As shown in Lemma \ref{lem:proper-1}, the optimal total rewards in $M$ satisfy $\bfx = A_\sigma \bfx+b_\sigma$ and in $M'$ statisfy $\bfx' = A'_{\sigma}\bfx'+b_{\sigma}$. 
Moreover, $\bfx$ and $\bfx'$ are the sole fixed points of these equations.
To avoid clutter, we will write $A$, $A'$ and $\bfb$ instead of $A_\sigma$, $A'_\sigma$ and $\bfb_\sigma$.
We know that $\|A-A'\|_1 \leq 2\varepsilon$ and let $\mathcal E = A'-A$.

When we try to converge at $\bfx'$ by iterating $\bfx'_{i+1} = A'\bfx'_i+b'$ when starting at $\bfx'_0 = \bfx$, we obtain the following:
\begin{align*}
\bfx'_1 &= A' \bfx_0 + \bfb = (A'-A)\bfx + A \bfx + \bfb = (A'-A)\bfx + \bfx = \mathcal E \bfx + \bfx_0 \\
\bfx'_2 &= A' \bfx_1 + \bfb = A' (\mathcal E \bfx + \bfx_0) + \bfb = A' \mathcal E \bfx + A' \bfx_0 + \bfb = A' \mathcal E \bfx + \bfx_1 \\
\bfx'_3 &= A' \bfx_2 + \bfb = A' (A' \mathcal E \bfx + \bfx_1) + \bfb = (A')^2 \mathcal E \bfx + A' \bfx_1 + \bfb = (A')^2 \mathcal E \bfx + \bfx_2 \\
&\vdots\\
\bfx'_{k+1} &= (A')^k \mathcal E \bfx + \bfx_k \\
\end{align*}
and thus $\bfx' - \bfx = (\lim_{k\to\infty} \bfx'_{k+1}) - \bfx_0 = (\lim_{k\to\infty} \sum_{i=0}^{k} {A'}^i \mathcal E \bfx + \bfx_0) - \bfx_0 = (\sum_{i=0}^\infty {A'}^i) \mathcal E \bfx$.

Note that the time to terminate in $M$ and $M'$ from each vertex is equal to $(\sum_{i=0}^\infty {A}^i) \bf 1$ and $(\sum_{i=0}^\infty {A'}^i) \bf 1$, respectively. We know that all of these values are $\leq K$. We know that the absolute value of each entry of $\bfx$ is at most $Kr_{max}$, 
as the most we can get in each step is $r_{max}$ during the expected number of $\leq K$ steps.
So $\|\mathcal E \bfx\|_{\infty} \leq 2\varepsilon K r_{max} {\bf 1}$,
and so $\|(\sum_{i=0}^\infty {A'}^i) \mathcal E \bfx\|_\infty \leq \|2\varepsilon K r_{max} (\sum_{i=0}^\infty {A'}^i) {\bf 1} \|_\infty \leq 2\varepsilon K^2 r_{max}$ as required.
\end{proof}

We observe that the estimate in Lemma \ref{lem:proper-bound} holds when we take the supremum over all stackless $\varepsilon$-proper strategies. This is because for any two functions $f$ and $g$ such that 
$f(\sigma)-g(\sigma)\leq \varepsilon$ for every stackless $\sigma$ 
and $\sigma' = \argmax_\sigma f(\sigma)$ 
we get $\sup_\sigma f(\sigma) - \sup_\sigma g(\sigma) \leq f(\sigma') - g(\sigma') \leq \varepsilon$. 
In other words,  Lemma \ref{lem:proper-bound} implies that 
$\|\ETotal^M - \ETotal^{M'}\|_\infty \leq 2 \epsilon K^2 r_{\max}$.

We can now choose $\varepsilon' =
\frac{\varepsilon}{2\epsilon K^2r_{\max}}$,
and learn $\delta_M$ up to precision $\varepsilon'$  
with probability $\geq \delta$ by the usual sampling techniques as done for finite MDPs.

We then obtain 1-exit RMDP $M'$, which we can solve efficiently and exactly using linear programming
similarly to what was done in \cite{EtessamiWY19}. First, let us create a variable $t_q$ for every $q \in Q$. The linear program is then as follows. 
\begin{align*}
&\text{Minimize} \sum_q t_q \text{ subject to:} &\\
t_q &\geq t_{en} + t_{(b,ex')} &\text{ if } q = (b, en) \in \Call, ex = \Ex_{Y(b)}\\
t_q &\geq r(q,a) + \sum\limits_{q' \in Q} p(q'|q, a) t_{q'} &\text{ otherwise, for every possible } a\in A \\
\end{align*}
One can easily show that any optimal solution gives us a fixed point of $F$ which we know has to be equal to $\ETotal^{M'}$.
This is because if all inequalities for a variable $t_q$ are non-strict then 
we can decrease the value of $t_q$ to the maximum of their right hand sides and still get a valid solution. As linear programs can be solved in polynomial time, and $\ETotal^{M'}$ is $\epsilon$-close to $\ETotal^{M}$, we obtain an efficient PAC learning algorithm.
\end{proof}

\end{document}